\documentclass[10pt,twocolumn,letterpaper]{article}

\usepackage[pagenumbers]{cvpr} 

\makeatletter
\@namedef{ver@everyshi.sty}{}
\makeatother

\usepackage{graphicx}
\usepackage{amsmath}
\usepackage{amssymb}
\usepackage{booktabs}

\usepackage{pseudocode}
\usepackage{algorithmicx}
\usepackage[ruled]{algorithm}
\usepackage{algpseudocode}
\usepackage{amsthm}

\usepackage{booktabs, caption, makecell}

\usepackage{threeparttable}

\newtheorem{definition}{Definition}
\usepackage{amsmath,amsfonts,amssymb,amsthm}

\usepackage{chngcntr}
\newtheorem{assumption}{Assumption}

\usepackage{amsfonts} 
\newtheorem{theorem}{Theorem}

\newtheorem{lemma}[theorem]{Lemma}
\newtheorem{proposition}{Proposition}
\newtheorem{remark}{Remark}

\newtheorem{assumptionalt}{Assumption}[assumption]

\newenvironment{assumptionp}[1]{
  \renewcommand\theassumptionalt{#1}
  \assumptionalt
}{\endassumptionalt}

\newtheorem{theoremalt}{Theorem}[theorem]

\newenvironment{theoremp}[1]{
  \renewcommand\thetheoremalt{#1}
  \theoremalt
}{\endtheoremalt}

\newtheorem{corollaryalt}{Corollary}

\newenvironment{corollaryp}[1]{
  \renewcommand\thecorollaryalt{#1}
  \corollaryalt
}{\endcorollaryalt}

\makeatletter
\AfterEndEnvironment{algorithm}{\let\@algcomment\relax}
\AtEndEnvironment{algorithm}{\kern2pt\hrule\relax\vskip3pt\@algcomment}
\let\@algcomment\relax
\newcommand\algcomment[1]{\def\@algcomment{\footnotesize#1}}

\renewcommand\fs@ruled{\def\@fs@cfont{\bfseries}\let\@fs@capt\floatc@ruled
  \def\@fs@pre{\hrule height.8pt depth0pt \kern2pt}%
  \def\@fs@post{}%
  \def\@fs@mid{\kern2pt\hrule\kern2pt}%
  \let\@fs@iftopcapt\iftrue}
\makeatother

\usepackage[pagebackref,breaklinks,colorlinks]{hyperref}
\usepackage{layouts}

\usepackage{progressbar}

\usepackage[capitalize]{cleveref}
\crefname{section}{Sec.}{Secs.}
\Crefname{section}{Section}{Sections}
\Crefname{table}{Table}{Tables}
\crefname{table}{Tab.}{Tabs.}

\def\cvprPaperID{10785} 
\def\confName{CVPR}
\def\confYear{2023}

\begin{document}

\title{Partial Variance Reduction improves Non-Convex Federated learning on heterogeneous data}


\author{Bo Li$^{\ast}$ \qquad Mikkel N. Schmidt\qquad Tommy S. Alstrøm \\
Technical University of Denmark\\
{\tt\small \{blia, mnsc, tsal\}@dtu.dk } \\
\and 
Sebastian U. Stich \\
CISPA$^{\ddagger}$ \\
{\tt\small stich@cispa.de}\\
}
\maketitle

\footnotetext{\textit{$^\ast$ Work done while at CISPA}}
\footnotetext{\textit{$^\ddagger$ CISPA Helmholtz Center for Information Security}}

\begin{abstract}

Data heterogeneity across clients is a key challenge in federated learning. Prior works address this by either aligning client and server models or using control variates to correct client model drift. Although these methods achieve fast convergence in convex or simple non-convex problems, the performance in over-parameterized models such as deep neural networks is lacking. In this paper, we first revisit the widely used FedAvg algorithm in a deep neural network to understand how data heterogeneity influences the gradient updates across the neural network layers. We observe that while the feature extraction layers are learned efficiently by FedAvg, the substantial diversity of the final classification layers across clients impedes the performance. Motivated by this, we propose to correct model drift by variance reduction only on the final layers. We demonstrate that this significantly outperforms existing benchmarks at a similar or lower communication cost. We furthermore provide proof for the convergence rate of our algorithm.

\end{abstract}

\section{Introduction}
\label{sec:intro}
Federated learning (FL) is emerging as an essential distributed learning paradigm in large-scale machine learning. Unlike in traditional machine learning, where a model is trained on the collected centralized data, in federated learning, each client (\eg phones and institutions) learns a model with its local data. A centralized model is then obtained by aggregating the updates from all participating clients without ever requesting the client data, thereby ensuring a certain level of user privacy~\cite{DBLP:journals/corr/abs-1912-04977, DBLP:journals/corr/KonecnyMRR16}. Such an algorithm is especially beneficial for tasks where the data is sensitive, \eg chemical hazards detection and diseases diagnosis ~\cite{Sheller2020}.

\begin{figure}[ht!]
    \centering
    \includegraphics[width=.48\textwidth]{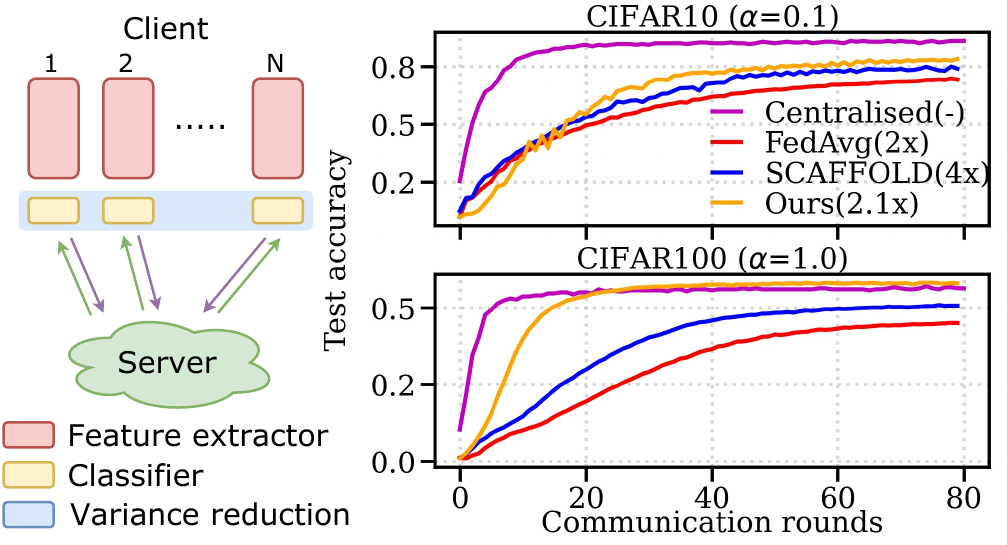}
    \caption{Our proposed FedPVR framework with the performance (communicated parameters per round client$\Longleftrightarrow$server). Smaller $\alpha$ corresponds to higher data heterogeneity. Our method achieves a better speedup than existing approaches by transmitting a slightly larger number of parameters than FedAvg.}
    \label{fig:concept}
\end{figure}

Two primary challenges in federated learning are i) handling data heterogeneity across clients~\cite{DBLP:journals/corr/abs-1912-04977} and ii) limiting the cost of communication between the server and clients~\cite{Halgamuge2009AnEO}. In this setting, FedAvg~\cite{DBLP:journals/corr/KonecnyMRR16} is one of the most widely used schemes: A server broadcasts its model to clients, which then update the model using their local data in a series of steps before sending their individual model to the server, where the models are aggregated by averaging the parameters. The process is repeated for multiple communication rounds. While it has shown great success in many applications, it tends to achieve subpar accuracy and convergence when the data are heterogeneous~\cite{DBLP:journals/corr/abs-1812-06127, DBLP:journals/corr/abs-1910-06378, DBLP:journals/corr/abs-2106-05001}. 

The slow and sometimes unstable convergence of Fed\-Avg can be caused by client drift~\cite{DBLP:journals/corr/abs-1910-06378} brought on by data heterogeneity. Numerous efforts have been made to improve FedAvg's performance in this setting. Prior works attempt to mitigate client drift by penalizing the distance between a client model and the server model~\cite{DBLP:journals/corr/abs-1812-06127, DBLP:journals/corr/abs-2103-16257} or by performing variance reduction techniques while updating client models~\cite{DBLP:journals/corr/abs-1910-06378, DBLP:journals/corr/abs-2111-04263, DBLP:journals/corr/ShamirS013}. These works demonstrate fast convergence on convex problems or for simple neural networks; however, their performance on deep neural networks, which are state-of-the-art for many centralized learning tasks~\cite{DBLP:journals/corr/HeZRS15, Simonyan15}, has yet to be well explored. Adapting techniques that perform well on convex problems to neural networks is non-trivial~\cite{DBLP:journals/corr/abs-1812-04529} due to their ``intriguing properties''~\cite{DBLP:journals/corr/SzegedyZSBEGF13} such as over-parametrization and permutation symmetries.

To overcome the above issues, we revisit the FedAvg algorithm with a deep neural network (VGG-11~\cite{Simonyan15}) under the assumption of data heterogeneity and full client participation. Specifically, we investigate which layers in a neural network are mostly influenced by data heterogeneity. We define \textit{drift diversity}, which measures the diversity of the directions and scales of the averaged gradients across clients per communication round. We observe that in the non-IID scenario, the deeper layers, especially the final classification layer, have the highest diversity across clients compared to an IID setting. This indicates that FedAvg learns good feature representations even in the non-IID scenario~\cite{https://doi.org/10.48550/arxiv.2205.13692} and that the significant variation of the deeper layers across clients is a primary cause of FedAvg's subpar performance.

Based on the above observations, we propose to align the classification layers across clients using variance reduction. Specifically, we estimate the average updating direction of the classifiers (the last several fully connected layers) at the client $\boldsymbol{c}_{i}$ and server level $\boldsymbol{c}$ and use their difference as a control variate~\cite{DBLP:journals/corr/abs-1910-06378} to reduce the variance of the classifiers across clients. We analyze our proposed algorithm and derive a convergence rate bound.

We perform experiments on the popular federated learning benchmark datasets CIFAR10~\cite{Krizhevsky2009LearningML} and CIFAR100~\cite{Krizhevsky2009LearningML} using two types of neural networks, VGG-11~\cite{Simonyan15} and ResNet-8~\cite{DBLP:journals/corr/HeZRS15}, and different levels of data heterogeneity across clients. We experimentally show that we require fewer communication rounds compared to the existing methods~\cite{DBLP:journals/corr/abs-1812-06127, DBLP:journals/corr/abs-1910-06378, DBLP:journals/corr/KonecnyMRR16} to achieve the same accuracy while transmitting a similar or slightly larger number of parameters between server and clients than FedAvg (see Fig.~\ref{fig:concept}). With a (large) fixed number of communication rounds, our method achieves on-par or better top-1 accuracy, and in some settings it even outperforms centralized learning. Using conformal prediction~\cite{angelopoulos2021uncertainty}, we show how performance can be improved further using adaptive prediction sets. 

We show that applying variance reduction on the last layers increases the diversity of the feature extraction layers. This diversity in the feature extraction layers may give each client more freedom to learn richer feature representations, and the uniformity in the classifier then ensures a less biased decision. We summarize our contributions here:
\begin{itemize}
    \item We present our algorithm for partial variance-reduced federated learning (FedPVR). We experimentally demonstrate that the key to the success of our algorithm is the diversity between the feature extraction layers and the alignment between the classifiers. 
    \item We prove the convergence rate in the convex settings and non-convex settings, precisely characterize its weak dependence on data-heterogeneity measures and show that FedPVR provably converges as fast as the centralized SGD baseline in most practical relevant cases. 
    \item We experimentally show that our algorithm is more communication efficient than previous works across various levels of data heterogeneity, datasets, and neural network architectures. In some cases where data heterogeneity exists, the proposed algorithm even performs slightly better than centralized learning.
\end{itemize}

\begin{figure*}[ht!]
    \centering
    \includegraphics{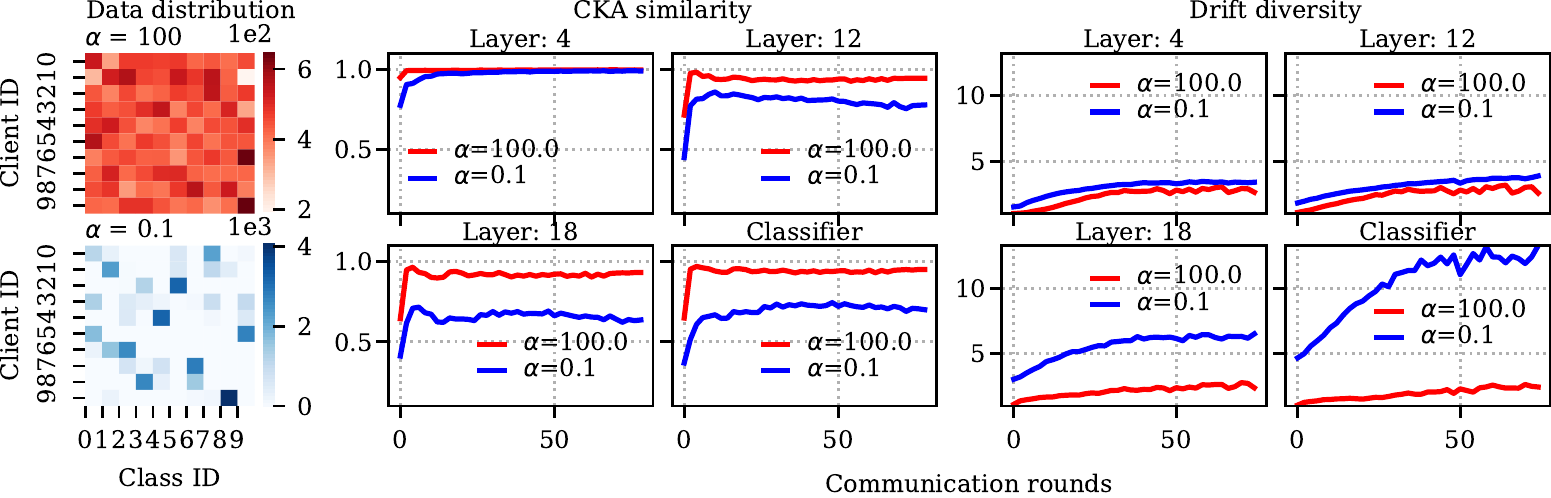}  
    \caption{Data distribution (number of images per client per class) with different levels of heterogeneity, client CKA similarity, and the drift diversity of each layer in VGG-11 (20 layers) with FedAvg. \textbf{Deep layers in an over-parameterised neural network have higher disagreement and variance when the clients are heterogeneous using FedAvg.}}
    \label{fig:cka_similarity}
\end{figure*}

\section{Related work}
\label{sec:related_work}
\subsection{Federated learning}
Federated learning (FL) is a fast-growing field~\cite{DBLP:journals/corr/abs-1912-04977, DBLP:journals/corr/abs-2107-06917}. We mainly describe FL methods in non-IID settings where the data is distributed heterogeneously across clients. Among the existing approaches, FedAvg~\cite{DBLP:journals/corr/McMahanMRA16} is the \textit{de facto} optimization technique. Despite its solid empirical performances in IID settings~\cite{DBLP:journals/corr/McMahanMRA16, DBLP:journals/corr/abs-1912-04977}, it tends to achieve a subpar accuracy-communication trade-off in non-IID scenarios.

Many works attempt to tackle FL when data is heterogeneous across clients~\cite{DBLP:journals/corr/abs-1812-06127, DBLP:journals/corr/abs-1910-06378, DBLP:conf/cvpr/GaoFLC0022, DBLP:journals/corr/abs-2111-04263, DBLP:journals/corr/abs-2108-04755, Varno2022MinimizingCD}. FedProx~\cite{DBLP:journals/corr/abs-1812-06127} proposes a temperature parameter and proximal regularization term to control the divergence between client and server models. However, the proximal term does not bring the alignment between the global and local optimal points~\cite{DBLP:journals/corr/abs-2111-04263}. 
Similarly, some works control the update direction by introducing client-dependent control variate~\cite{DBLP:journals/corr/abs-1910-06378, DBLP:journals/corr/ShamirS013, DBLP:journals/corr/KonecnyMRR16, DBLP:journals/corr/abs-2111-04263, mishchenko2022proxskip} that is also communicated between the server and clients. They have achieved a much faster convergence rate, but their performance in a non-convex setup, especially in deep neural networks, such as ResNet~\cite{DBLP:journals/corr/HeZRS15} and VGG~\cite{Simonyan15}, is not well explored. Besides, they suffer from a higher communication cost due to the transmission of the extra control variates, which may be a critical issue for resources-limited IoT mobile devices~\cite{Halgamuge2009AnEO}. Among these methods, SCAFFOLD~\cite{DBLP:journals/corr/abs-1910-06378} is the most closely related method to ours, and we give a more detailed comparison in section~\ref{sec:method} and~\ref{sec:experimental_results}.

Another line of work develops FL algorithms based on the characteristics, such as expressive feature representations~\cite{DBLP:journals/corr/abs-2010-15327} of neural networks. Collins et al.~\cite{https://doi.org/10.48550/arxiv.2205.13692} show that FedAvg is powerful in learning common data representations from clients' data. FedBabu~\cite{DBLP:journals/corr/abs-2106-06042}, TCT~\cite{https://doi.org/10.48550/arxiv.2207.06343}, and CCVR~\cite{DBLP:journals/corr/abs-2106-05001} propose to improve FL performance by finetuning the classifiers with a standalone dataset or features that are simulated based on the client models. However, preparing a standalone dataset/features that represents the data distribution across clients is challenging as this usually requires domain knowledge and may raise privacy concerns. Moon~\cite{DBLP:journals/corr/abs-2103-16257} encourages the similarity of the representations across different client models by using contrastive loss~\cite{DBLP:journals/corr/abs-2002-05709} but with the cost of three full-size models in memory on each client, which may limit its applicability in resource-limited devices.  

Other works focus on reducing the communication cost by compressing the transmitted gradients~\cite{https://doi.org/10.48550/arxiv.2002.11364, DBLP:journals/corr/abs-1901-09269, DBLP:journals/corr/abs-1911-08250, DBLP:journals/corr/Alistarh0TV16, StichCJ18sparseSGD}. They can reduce the communication bandwidth by adjusting the number of bits sent per iteration. These works are complementary to ours and can be easily integrated into our method to save communication costs.

\subsection{Variance reduction}

Stochastic variance reduction (SVR), such as SVRG~\cite{Johnson2013AcceleratingSG}, SAGA~\cite{DBLP:journals/corr/DefazioBL14}, and their variants, use control variate to reduce the variance of traditional stochastic gradient descent (SGD). These methods can remarkably achieve a linear convergence rate for strongly convex optimization problems compared to the sub-linear rate of SGD. Many federated learning algorithms, such as SCAFFOLD~\cite{DBLP:journals/corr/abs-1910-06378} and DANE~\cite{DBLP:journals/corr/ShamirS013}, have adapted the idea of variance reduction for the whole model and achieved good convergence on convex problems. However, as~\cite{DBLP:journals/corr/abs-1812-04529} demonstrated, naively applying variance reduction techniques gives no actual variance reduction and tends to result in a slower convergence in deep neural networks.
This suggests that adapting SVR techniques in deep neural networks for FL requires a more careful design.

\subsection{Conformal prediction}
Conformal prediction is a general framework that computes a prediction set guaranteed to include the true class with a high user-determined probability~\cite{angelopoulos2021uncertainty, 10.5555/3495724.3496026}. It requires no retraining of the models and achieves a finite-sum coverage guarantee~\cite{angelopoulos2021uncertainty}. As FL algorithms can hardly perform as well as centralized learning~\cite{DBLP:journals/corr/abs-2106-05001} when the data heterogeneity is high, we can integrate conformal prediction in FL to improve the empirical coverage by slightly increasing the predictive set size. This can be beneficial in sensitive use cases such as detecting chemical hazards, where it is better to give a prediction set that contains the correct class than producing a single but wrong prediction.

\section{Method}
\label{sec:method}

\subsection{Problem statement}
Given $N$ clients with full participation, we formalise the problem as minimizing the average of the stochastic functions with access to stochastic samples in Eq.~\ref{eq:problem} where $\boldsymbol{x}$ is the model parameters and $f_i$ represents the loss function at client $i$ with dataset $\mathcal{D}_i$,
\begin{equation}
    \min_{\boldsymbol{x}\in\mathbb{R}^d}\left(f(\boldsymbol{x}):=\frac{1}{N}\sum_{i=1}^N f_i(\boldsymbol{x})\right),
    \label{eq:problem}
\end{equation}
where $f_i(\boldsymbol{x}) := \mathbb{E}_{\mathcal{D}_i}[f_i(\boldsymbol{x};\mathcal{D}_i)]$.

\begin{table}[tb]
\vspace*{-3mm}
\caption{Notations used in this paper}
\vspace*{-3mm}
\resizebox{\linewidth}{!}{
\begin{tabular}{ll}
\toprule
$R, r$ & Number of communication rounds and round index  \\ 
$K, k$ & Number of local steps, local step index \\
$N, i$ & Number of clients, client index \\
$\boldsymbol{y}_{i,k}^r$ & client model $i$ at step $k$ and round $r$ \\
$\boldsymbol{x}^{r}$ & server model at round $r$  \\ 
$\boldsymbol{c}_i^r$, $\boldsymbol{c}^r$ & client and server control variate \\
\bottomrule
\end{tabular}}
\vspace*{-0.5cm}
\end{table}

\subsection{Motivation}
When the data $\{\mathcal{D}_i\}$ are heterogeneous across clients, FedAvg suffers from \emph{client drift}~\cite{DBLP:journals/corr/abs-1910-06378}, where the average of the local optimal $\bar{\boldsymbol{x}}^*=\frac{1}{N}\sum_{i\in N}\boldsymbol{x}_i^*$ is far from the global optimal $\boldsymbol{x}^*$. To understand what causes client drift, specifically which layers in a neural network are influenced most by the data heterogeneity, we perform a simple experiment using FedAvg and CIFAR10 datasets on a VGG-11. The detailed experimental setup can be found in section~\ref{sec:experimental_setup}. 

In an over-parameterized model, it is difficult to directly calculate client drift $||\bar{\boldsymbol{x}}^* - \boldsymbol{x}^*||^2$ as it is challenging to obtain the global optimum $\boldsymbol{x}^*$. We instead hypothesize that we can represent the influence of data heterogeneity on the model by measuring 1) drift diversity and 2) client model similarity. Drift diversity reflects the diversity in the amount each client model deviates from the server model after an update round.

\vspace*{-0.3cm}
\begin{definition}[Drift diversity]
We define the drift diversity across $N$ clients at round $r$ as:
\begin{equation}
    \xi^r := \frac{\sum_{i=1}^N||\boldsymbol{m}_i^r||^2}{||\sum_{i=1}^N \boldsymbol{m}_i^r||^2} \quad \boldsymbol{m}_i^r = \boldsymbol{y}_{i,K}^r -  \boldsymbol{x}^{r-1} 
    \label{eq:drift_diversity}
\end{equation}
\vspace*{-0.4cm}
\end{definition}
Drift diversity $\xi$ is high when all the clients update their models in different directions, i.e., when dot products between client updates $\boldsymbol{m}_i$ are small. When each client performs $K$ steps of vanilla SGD updates, $\xi$ depends on the directions and amplitude of the gradients over $N$ clients and is equivalent to $\frac{\sum_{i=1}^{N}||\sum_{k}g_i(\boldsymbol{y}_{i,k})||^2}{||\sum_{i=1}^N\sum_{k}g_i(\boldsymbol{y}_{i,k})||^2}$, where $g_i(\boldsymbol{y}_{i,k})$ is the stochastic mini-batch gradient.

After updating client models, we quantify the client model similarity using centred kernel alignment (CKA)~\cite{DBLP:journals/corr/abs-1905-00414} computed on a test dataset. CKA is a widely used permutation invariant metric for measuring the similarity between feature representations in neural networks~\cite{DBLP:journals/corr/abs-2106-05001,DBLP:journals/corr/abs-1905-00414, DBLP:journals/corr/abs-2010-15327}.

Fig.~\ref{fig:cka_similarity} shows the movement of $\xi$ and CKA across different levels of data heterogeneity using FedAvg. We observe that the similarity and diversity of the early layers (\eg layer index 4 and 12) are with a higher agreement between the IID ($\alpha=100.0$) and non-IID ($\alpha=0.1$) experiments, which indicates that FedAvg can still learn and extract good feature representations even when it is trained with non-IID data. The lower similarity on the deeper layers, especially the classifiers, suggests that these layers are strongly biased towards their local data distribution. When we only look at the model that is trained with $\alpha=0.1$, we see the highest diversity and variance on the classifiers across clients compared to the rest of the layers. Based on the above observations, we propose to align the classifiers across clients using variance reduction. We deploy client and server control variates to control the updating directions of the classifiers.

\subsection{Classifier variance reduction}
Our proposed algorithm (Alg.~\textcolor{red}{I}) consists of three parts: i) client updating (Eq.~\ref{eq:client_update_svr}-\ref{eq:client_update_sgd}) ii) client control variate updating, (Eq.~\ref{eq:control_variate_update}), and iii) server updating (Eq.~\ref{eq:server_update_x}-\ref{eq:server_update_c})

We first define a vector $\boldsymbol{p} \in \mathbb{R}^d$ that contains $0$ or $1$ with $v$ non-zero elements ($v\ll d$) in Eq.~\ref{eq:p_define}. We recover SCAFFOLD with $\boldsymbol{p} = \boldsymbol{1}$ and recover FedAvg with $\boldsymbol{p} = \boldsymbol{0}$. For the set of indices $j$ where $\boldsymbol{p}_j=1$ ($S_{\text{svr}}$ from Eq.~\ref{eq:indices}), we update the corresponding weights $\boldsymbol{y}_{i, S_{\text{svr}}}$ with variance reduction such that we maintain a state for each client ($\boldsymbol{c}_i \in \mathbb{R}^v$) and for the server ($\boldsymbol{c} \in \mathbb{R}^v$) in Eq.~\ref{eq:client_update_svr}. For the rest of the indices $S_{\text{sgd}}$ from Eq.~\ref{eq:indices}, we update the corresponding weights $\boldsymbol{y}_{i S_{\text{sgd}}}$ with SGD in Eq.~\ref{eq:client_update_sgd}. As the server variate $\boldsymbol{c}$ is an average of $\boldsymbol{c}_i$ across clients, we can safely initialise them as $\boldsymbol{0}$.

In each communication round, each client receives a copy of the server model $\boldsymbol{x}$ and the server control variate $\boldsymbol{c}$. They then perform $K$ model updating steps (see Eq.~\ref{eq:client_update_svr}-~\ref{eq:client_update_sgd} for one step) using cross-entropy as the loss function. Once this is finished, we calculate the updated client control variate $\boldsymbol{c}_i$ using Eq.~\ref{eq:control_variate_update}. The server then receives the updated $\boldsymbol{c}_{i}$ and $\boldsymbol{y}_{i}$ from all the clients for aggregation (Eq.~\ref{eq:server_update_x}-\ref{eq:server_update_c}). This completes one communication round.

\vspace*{-5mm}
\begin{flalign}
\boldsymbol{p} &:= \{0, 1\}^d, \quad v = \sum \boldsymbol{p} \label{eq:p_define}\\
S_{\text{svr}} &:= \{j: \boldsymbol{p}_j = 1\}, \quad S_{\text{sgd}} := \{j: \boldsymbol{p}_j = 0\} \label{eq:indices}\\
\boldsymbol{y}_{i, S_{\text{svr}}} &\leftarrow \boldsymbol{y}_{i, S_{\text{svr}}} - \eta_l (g_i(\boldsymbol{y}_i)_{S_{\text{svr}}} - \boldsymbol{c}_i + \boldsymbol{c}) \label{eq:client_update_svr}\\
\boldsymbol{y}_{i, S_{\text{sgd}}} &\leftarrow \boldsymbol{y}_{i, S_{\text{sgd}}} - \eta_l g_i(\boldsymbol{y}_i)_{S_{\text{sgd}}} \label{eq:client_update_sgd}\\
\boldsymbol{c}_i &\leftarrow \boldsymbol{c}_i - \boldsymbol{c} + \frac{1}{K\eta_l}\left(\boldsymbol{x}_{S_{\text{svr}}} - \boldsymbol{y}_{i, S_{\text{svr}}}\right)
\label{eq:control_variate_update} \\
\boldsymbol{x} &\leftarrow (1-\eta_g)\boldsymbol{x} + \frac{1}{N}\sum_{i\in N}\boldsymbol{y}_{i}
\label{eq:server_update_x}\\
\boldsymbol{c} &\leftarrow \frac{1}{N}\sum_{i\in N}\boldsymbol{c}_i
\label{eq:server_update_c}
\end{flalign}
\vspace*{-5mm}

\begin{algorithm}[ht!]
\floatname{algorithm}{Algorithm I}
\small
\caption{Partial variance reduction (FedPVR)}
\algcomment{In terms of implementation, we can simply assume the control variate for the block of weights that is updated with SGD as $\boldsymbol{0}$ and implement line 8 and 9 in one step}
\label{alg}

\hspace*{\algorithmicindent} \textbf{server}: initialise the server model $\boldsymbol{x}$, the control variate $\boldsymbol{c}$, and global step size $\eta_g$

\hspace*{\algorithmicindent} \textbf{client}: initialise control variate $\boldsymbol{c}_i$ and local step size $\eta_l$

\hspace*{\algorithmicindent} \textbf{mask}: $\boldsymbol{p}:=\{0, 1\}^d$, $S_{\text{sgd}}:=\{j: \boldsymbol{p}_j = 0\}$, $S_{\text{svr}}:= \{j: \boldsymbol{p}_j = 1\}$

\label{Algorithm:l_scaffold}
\begin{algorithmic}[1]
\Procedure{Model updating}{}
    \For {$r = 1 \to R $}
    \State \textbf{communicate} $\boldsymbol{x}$\ \textbf{and} $\boldsymbol{c}$ \textbf{to all clients} $i \in [N]$
    \For {On client $i \in [N]$ in parallel}
    \State $\boldsymbol{y}_i \leftarrow \boldsymbol{x}$
    \For {$k = 1 \to K$}
        \State compute minibatch gradient $g_i(\boldsymbol{y}_{i})$
        \State $\boldsymbol{y}_{i, S_{\text{sgd}}} \leftarrow \boldsymbol{y}_{i, S_{\text{sgd}}} - \eta_l g_i(\boldsymbol{y}_i)_{S_{\text{sgd}}}$
        \State $\boldsymbol{y}_{i, S_{\text{svr}}} \leftarrow \boldsymbol{y}_{i, S_{\text{svr}}} - \eta_l (g_i(\boldsymbol{y}_i)_{S_{\text{svr}}} - \boldsymbol{c}_i+\boldsymbol{c})$ 
    \EndFor
    \State $\boldsymbol{c}_{i} \leftarrow \boldsymbol{c}_{i} - \boldsymbol{c} + \frac{1}{K\eta_l}(\boldsymbol{x}_{S_{\text{svr}}} - \boldsymbol{y}_{i, S_{\text{svr}}})$
    \State \textbf{communicate} $\boldsymbol{y}_{i}, \boldsymbol{c}_{i}$
    \EndFor
    \State $\boldsymbol{x} \leftarrow (1-\eta_g)\boldsymbol{x} + \frac{1}{N}\sum_{i\in N}\boldsymbol{y}_{i}$
    \State $\boldsymbol{c} \leftarrow \frac{1}{N}\sum_{i\in N}\boldsymbol{c}_i$
    \EndFor
\EndProcedure
\Statex
\end{algorithmic}
  \vspace{-0.4cm}%
\end{algorithm}
\vspace*{-5mm}

\textbf{Ours vs SCAFFOLD~\cite{DBLP:journals/corr/abs-1910-06378}} While our work is similar to SCAFFOLD in the use of variance reduction, there are some fundamental differences. We both communicate control variates between the clients and server, but our control variate ($2v \leq 0.1d$ ) is significantly smaller than the one in SCAFFOLD ($2d$). This 2x decrease in bits can be critical for some low-power IoT devices as the communication may consume more energy~\cite{Halgamuge2009AnEO}. From the application point of view, SCAFFOLD achieved great success in convex or simple two layers problems. 
However, adapting the techniques that work well from convex problems to over-parameterized models is non-trivial~\cite{DBLP:journals/corr/SzegedyZSBEGF13}, and naively adapting variance reduction techniques on deep neural networks gives little or no convergence speedup~\cite{DBLP:journals/corr/abs-1812-04529}. Therefore, the significant improvement achieved by our method gives essential and non-trivial insight into what matters when tackling data heterogeneity in FL in over-parameterized models. 

\subsection{Convergence rate}
We state the convergence rate in this section. We assume functions $\{f_i\}$ are $\beta$-smooth following~\cite{DBLP:journals/corr/abs-1907-04232, DBLP:journals/corr/abs-2003-10422}. We then assume $g_i(\boldsymbol{x}):=\nabla f_i(x;\mathcal{D}_i)$ is an unbiased stochastic gradient of $f_i$ with variance bounded by $\sigma^2$. We assume strongly convexity ($\mu > 0$) and general convexity ($\mu=0$) for some of the results following~\cite{DBLP:journals/corr/abs-1910-06378}. Furthermore, we also make assumptions about the heterogeneity of the functions.  

For convex functions, we assume the heterogeneity of the function $\{f_i\}$ at the optimal point $\boldsymbol{x}^*$ (such a point always exists for a strongly convex function) following~\cite{DBLP:journals/corr/abs-2003-10422, DBLP:journals/corr/abs-1909-04746}. 
\vspace*{-2mm}

\begin{assumption}[$\zeta$-heterogeneity]
We define a measure of variance at the optimum $\boldsymbol{x}^*$ given $N$ clients as :
\vspace*{-2mm}
\begin{equation}
\zeta^2 := \frac{1}{N}\sum_{i=1}^N\mathbb{E}||\nabla f_i(\boldsymbol{x}^*)||^2 \,.
\end{equation}
\label{assump:zeta}
\end{assumption}
\vspace*{-6mm}
For the non-convex functions, such an unique optimal point $\boldsymbol{x}^*$ does not necessarily exist, so we generalize Assumption~\ref{assump:zeta} to Assumption~\ref{assump:zeta_hat}.

\begin{assumption}[$\hat{\zeta}$-heterogeneity]
We assume there exists constant $\hat{\zeta}$ such that $\forall \boldsymbol{x}\in \mathbb{R}^d$
\vspace*{-2mm}
\begin{equation}
\frac{1}{N}\sum_{i=1}^N\mathbb{E}||\nabla f_i(\boldsymbol{x})||^2 \leq \hat{\zeta}^2 \,.
\end{equation}
\label{assump:zeta_hat}
\end{assumption}
\vspace*{-4mm}

Given the mask $\boldsymbol{p}$ as defined in Eq.~\ref{eq:p_define}, we know $||\boldsymbol{p}\odot\boldsymbol{x}|| \leq ||\boldsymbol{x}||$. Therefore, we have the following propositions.

\begin{proposition}[Implication of Assumption~\ref{assump:zeta}]
Given the mask $\boldsymbol{p}$, we define the heterogeneity of the block of weights that are not variance reduced at the optimum $\boldsymbol{x}^*$ as:
\vspace*{-2mm}
\begin{equation}
    \zeta_{1-p}^2 := \frac{1}{N}\sum_{i=1}^{N}||(\boldsymbol{1}-\boldsymbol{p})\odot\nabla f_i(\boldsymbol{x^*})||^2 , 
\end{equation}
If Assumption~\ref{assump:zeta} holds, then it also holds that:
\vspace*{-2mm}
\begin{equation}
    \zeta_{1-p}^2 \leq \zeta^2\,.
\end{equation}
\label{proposition:zeta}
\end{proposition}
\vspace*{-6mm}

In Proposition~\ref{proposition:zeta}, $\zeta_{1-p}^2 = \zeta^2$ if $\boldsymbol{p} = \boldsymbol{0}$ and $\zeta_{1-p}^2 = 0$ if $\boldsymbol{p}=\boldsymbol{1}$. If $\boldsymbol{p} \neq \boldsymbol{0}$ and $\boldsymbol{p} \neq \boldsymbol{1}$, as the heterogeneity of the shallow weights is lower than the deeper weights~\cite{https://doi.org/10.48550/arxiv.2207.06343}, we have $\zeta_{1-p}^2 \leq \zeta^2$. Similarly, we can validate Proposition~\ref{proposition:zeta_hat}.

\begin{proposition}[Implication of Assumption~\ref{assump:zeta_hat}]
Given the mask $\boldsymbol{p}$, we assume there exists constant $\hat{\zeta}_{1-p}$ such that $\forall \boldsymbol{x} \in \mathbb{R}^d$, the heterogeneity of the block of weights that are not variance reduced:
\vspace*{-2mm}
\begin{equation}
    \frac{1}{N}\sum_{i=1}^{N}||(\boldsymbol{1}-\boldsymbol{p})\odot \nabla f_i(\boldsymbol{x})||^2  \leq \hat{\zeta}_{1-p} , 
\end{equation}
If Assumption~\ref{assump:zeta_hat} holds, then it also holds that:
\vspace*{-2mm}
\begin{equation}
    \hat{\zeta}_{1-p}^2 \leq \hat{\zeta}^2\,.
\end{equation}
\label{proposition:zeta_hat}
\end{proposition}
\vspace*{-10mm}

\begin{theorem}
\label{theorem_convergence}
For any $\beta$-smooth function $\{f_i\}$, the output of \texttt{FedPVR} has expected error smaller than $\epsilon$ for $\eta_g = \sqrt{N}$ and some values of $\eta_l$, $R$ satisfying:
\begin{itemize}
    \item \textbf{Strongly convex:} $\eta_l \leq \min\left(\frac{1}{80K\eta_g\beta}, \frac{26}{20\mu K\eta_g}\right)$,
    \begin{equation}
        R = \tilde{\mathcal{O}}\left(\frac{\sigma^2}{\mu NK\epsilon} + \frac{\zeta_{1-p}^2}{\mu\epsilon} + \frac{\beta}{\mu}\right),
    \end{equation}
    \vspace*{-6mm}
    \item \textbf{General convex:} $\eta_l \leq \frac{1}{80K\eta_g\beta}$,
    \begin{equation}
        R = \mathcal{O}\left(\frac{\sigma^2D}{KN\epsilon^2} + \frac{\zeta_{1-p}^2D}{\epsilon^2} + \frac{\beta D}{\epsilon}  + F\right),
    \end{equation}
    \vspace*{-6mm}
    \item \textbf{Non-convex}: $\eta_l \leq \frac{1}{26K\eta_g\beta}$, and $R\geq 1$, then:
    \begin{equation}
        R=\mathcal{O}\left(\frac{\beta\sigma^2F}{KN\epsilon^2} + \frac{\beta\hat{\zeta}_{1-p}^2F}{N\epsilon^2} + \frac{\beta F}{\epsilon} \right),
    \end{equation}
\end{itemize}
\vspace*{-4mm}
Where $D := ||\boldsymbol{x}^0 - \boldsymbol{x}^*||^2$ and $F:=f(\boldsymbol{x}^0) - f^*$.
\end{theorem}

Given the above assumptions, the convergence rate is given in Theorem~\ref{theorem_convergence}. When $\boldsymbol{p} = \boldsymbol{1}$, we recover SCAFFOLD convergence guarantee as $\zeta_{1-p}^2=0, \hat{\zeta}_{1-p}^2=0$. In the strongly convex case, the effect of the heterogeneity of the block of weights that are not variance reduced $\zeta_{1-p}^2$ becomes negligible if $\tilde{\mathcal{O}}\left(\frac{\zeta_{1-p}^2}{\epsilon}\right)$ is sufficiently smaller than $\tilde{\mathcal{O}}\left(\frac{\sigma^2}{NK\epsilon} \right)$. In such case, our rate is $\frac{\sigma^2}{NK\epsilon} + \frac{1}{\mu}$, which recovers the SCAFFOLD in the strongly convex without sampling and further matches that of SGD (with mini-batch size $K$ on each worker). We also recover the FedAvg rate\footnotemark{} at simple IID case. See Appendix.~\textcolor{red}{B} for the full proof.

\footnotetext{FedAvg at strongly convex case has the rate $R=\tilde{\mathcal{O}}(\frac{\sigma^2}{\mu KN\epsilon}  + \frac{\sqrt{\beta}G}{\mu\sqrt{\epsilon}} + \frac{\beta}{\mu})$ with $G$ measures the gradient dissimilarity. At simple IID case, G=0~\cite{DBLP:journals/corr/abs-1910-06378}.}
\section{Experimental setup}
\label{sec:experimental_setup}

We demonstrate the effectiveness of our approach with CIFAR10~\cite{Krizhevsky2009LearningML} and CIFAR100~\cite{Krizhevsky2009LearningML} on image classification tasks. We simulate the data heterogeneity scenario following~\cite{DBLP:conf/nips/LinKSJ20} by partitioning the data according to the Dirichlet distribution with the concentration parameter $\alpha$. The smaller the $\alpha$ is, the more imbalanced the data are distributed across clients. An example of the data distribution over multiple clients using the CIFAR10 dataset can be seen in Fig.~\ref{fig:cka_similarity}. In our experiment, we use $\alpha\in\{0.1, 0.5, 1.0\}$ as these are commonly used concentration parameters~\cite{DBLP:conf/nips/LinKSJ20}. Each client has its local data, and this data is kept to be the same during all the communication rounds. We hold out the test dataset at the server for evaluating the classification performance of the server model. Following~\cite{DBLP:conf/nips/LinKSJ20}, we perform the same data augmentation for all the experiments.

We use two models:\ VGG-11 and ResNet-8 following~\cite{DBLP:conf/nips/LinKSJ20}. We perform variance reduction for the last three layers in VGG-11 and the last layer in ResNet-8. We use 10 clients with full participation following~\cite{https://doi.org/10.48550/arxiv.2207.06343} (close to cross-silo setup) and a batch size of 256. Each client performs 10 local epochs of model updating. We set the server learning rate $\eta_g=1$ for all the models~\cite{DBLP:journals/corr/abs-1910-06378}. We tune the clients learning rate from $\{0.05, 0.1, 0.2, 0.3\}$ for each individual experiment. The learning rate schedule is experimentally chosen from constant, cosine decay~\cite{DBLP:journals/corr/LoshchilovH16a}, and multiple step decay~\cite{DBLP:conf/nips/LinKSJ20}. We compare our method with the representative federated learning algorithms FedAvg~\cite{DBLP:journals/corr/McMahanMRA16}, FedProx~\cite{DBLP:journals/corr/abs-1812-06127}, SCAFFOLD~\cite{DBLP:journals/corr/abs-1910-06378}, and FedDyn~\cite{DBLP:journals/corr/abs-2111-04263}. 
All the results are averaged over three repeated experiments with different random initialization. We leave $1\%$ of the training data from each client out as the validation data to tune the hyperparameters (learning rate and schedule) per client. See Appendix.~\textcolor{red}{C} for additional experimental setups. The code is at \url{github.com/lyn1874/fedpvr}. 
\vspace*{-0.18cm}

\section{Experimental results}
\label{sec:experimental_results}
\begin{table*}[ht!]
\caption{The required number of communication rounds (speedup compared to FedAvg) to achieve a certain level of top-1 accuracy ($66\%$ for the CIFAR10 dataset and $44\%$ for the CIFAR100 dataset). Our method requires fewer rounds to achieve the same accuracy.}
\vspace*{-3mm}
\label{tab:communication_round}
\resizebox{0.72\textwidth}{!}
{\begin{minipage}{\textwidth}
\begin{tabular}{@{}lcccccccc@{}}
\toprule
 & \multicolumn{4}{c}{CIFAR10 (66$\%$)} & \multicolumn{4}{c}{CIFAR100 (44$\%$)} \\ \cmidrule(lr){2-5}\cmidrule(l){6-9}
 & \multicolumn{2}{c}{$\alpha$=0.1} & \multicolumn{2}{c}{$\alpha$=0.5} & \multicolumn{2}{c}{$\alpha$=0.1} & \multicolumn{2}{c}{$\alpha$=1.0} \\ 
 \cmidrule(lr){2-3}\cmidrule(lr){4-5}\cmidrule(lr){6-7}\cmidrule(l){8-9}
 & VGG-11 & ResNet-8 & VGG-11 & ResNet-8 & VGG-11 & ResNet-8 & VGG-11 & ResNet-8 \\
 \midrule
 & No. rounds & No. rounds & No. rounds & No. rounds & No. rounds & No. rounds & No. rounds & No. rounds \\
FedAvg &\progressbar[linecolor=white, filledcolor=black, ticksheight=0.0, heighta=8pt, width=2.5em]{0.55}{\hspace{0.1cm}\raisebox{1.5pt}{$55(1.0\text{x})$}}
& \progressbar[linecolor=white, filledcolor=black, ticksheight=0.0, heighta=8pt, width=2.5em]{0.9}{\hspace{0.1cm}\raisebox{1.5pt}{$90(1.0\text{x})$}} &\progressbar[linecolor=white, filledcolor=black, ticksheight=0.0, heighta=8pt, width=2.5em]{0.15}{\hspace{0.1cm}\raisebox{1.5pt}{$15(1.0\text{x})$}}  &\progressbar[linecolor=white, filledcolor=black, ticksheight=0.0, heighta=8pt, width=2.5em]{0.15}{\hspace{0.1cm}\raisebox{1.5pt}{$15(1.0\text{x})$}}  &\progressbar[linecolor=white, filledcolor=black, ticksheight=0.0, heighta=8pt, width=2.5em]{1.0}{\hspace{0.1cm}\raisebox{1.5pt}{$100+(1.0\text{x})$}}  
&\progressbar[linecolor=white, filledcolor=black, ticksheight=0.0, heighta=8pt, width=2.5em]{1.0}{\hspace{0.1cm}\raisebox{1.5pt}{$100+(1.0\text{x})$}} &\progressbar[linecolor=white, filledcolor=black, ticksheight=0.0, heighta=8pt, width=2.5em]{0.8}{\hspace{0.1cm}\raisebox{1.5pt}{$80(1.0\text{x})$}}  &\progressbar[linecolor=white, filledcolor=black, ticksheight=0.0, heighta=8pt, width=2.5em]{0.56}{\hspace{0.1cm}\raisebox{1.5pt}{$56(1.0\text{x})$}} \\
FedProx &\progressbar[linecolor=white, filledcolor=black, ticksheight=0.0, heighta=8pt, width=2.5em]{0.52}{\hspace{0.1cm}\raisebox{1.5pt}{$52(1.1\text{x})$}}  &
\progressbar[linecolor=white, filledcolor=black, ticksheight=0.0, heighta=8pt, width=2.5em]{0.75}{\hspace{0.1cm}\raisebox{1.5pt}{$75(1.2\text{x})$}}&
\progressbar[linecolor=white, filledcolor=black, ticksheight=0.0, heighta=8pt, width=2.5em]{0.16}{\hspace{0.1cm}\raisebox{1.5pt}{$16(0.9\text{x})$}}  &
\progressbar[linecolor=white, filledcolor=black, ticksheight=0.0, heighta=8pt, width=2.5em]{0.2}{\hspace{0.1cm}\raisebox{1.5pt}{$20(0.8\text{x})$}}&
\progressbar[linecolor=white, filledcolor=black, ticksheight=0.0, heighta=8pt, width=2.5em]{1.0}{\hspace{0.1cm}\raisebox{1.5pt}{$100+(1.0\text{x})$}}  &
\progressbar[linecolor=white, filledcolor=black, ticksheight=0.0, heighta=8pt, width=2.5em]{1.0}{\hspace{0.1cm}\raisebox{1.5pt}{$100+(1.0\text{x})$}}&
\progressbar[linecolor=white, filledcolor=black, ticksheight=0.0, heighta=8pt, width=2.5em]{0.8}{\hspace{0.1cm}\raisebox{1.5pt}{$80(1.0\text{x})$}}&
\progressbar[linecolor=white, filledcolor=black, ticksheight=0.0, heighta=8pt, width=2.5em]{0.59}{\hspace{0.1cm}\raisebox{1.5pt}{$59(0.9\text{x})$}}  \\
SCAFFOLD & \progressbar[linecolor=white, filledcolor=black, ticksheight=0.0, heighta=8pt, width=2.5em]{0.39}{\hspace{0.1cm}\raisebox{1.5pt}{$39(1.4\text{x})$}} & \progressbar[linecolor=white, filledcolor=black, ticksheight=0.0, heighta=8pt, width=2.5em]{0.57}{\hspace{0.1cm}\raisebox{1.5pt}{$57(1.6\text{x})$}} &\progressbar[linecolor=white, filledcolor=black, ticksheight=0.0, heighta=8pt, width=2.5em]{0.14}{\hspace{0.1cm}\raisebox{1.5pt}{$14(1.0\text{x})$}}  & \progressbar[linecolor=white, filledcolor=black, ticksheight=0.0, heighta=8pt, width=2.5em]{0.09}{\hspace{0.1cm}\raisebox{1.5pt}{\hspace{4pt}$9(1.7\text{x})$}} &\progressbar[linecolor=white, filledcolor=black, ticksheight=0.0, heighta=8pt, width=2.5em]{0.8}{\hspace{0.1cm}\raisebox{1.5pt}{\hspace{7.5pt}$80(>1.3\text{x})$}}  &\progressbar[linecolor=white, filledcolor=black, ticksheight=0.0, heighta=8pt, width=2.5em]{0.61}{\hspace{0.1cm}\raisebox{1.5pt}{\hspace{4.3pt}$\textcolor{red}{\mathbf{61(>1.6\text{x})}}$}}  &\progressbar[linecolor=white, filledcolor=black, ticksheight=0.0, heighta=8pt, width=2.5em]{0.36}{\hspace{0.1cm}\raisebox{1.5pt}{$36(2.2\text{x})$}}  & \progressbar[linecolor=white, filledcolor=black, ticksheight=0.0, heighta=8pt, width=2.5em]{0.41}{\hspace{0.1cm}\raisebox{1.5pt}{$25(2.2\text{x})$}} \\
FedDyn & \progressbar[linecolor=white, filledcolor=black, ticksheight=0.0, heighta=8pt, width=2.5em]{0.27}{\hspace{0.1cm}\raisebox{1.5pt}{$\textcolor{red}{\textbf{27(2.0\text{x})}}$}} & \progressbar[linecolor=white, filledcolor=black, ticksheight=0.0, heighta=8pt, width=2.5em]{0.67}{\hspace{0.1cm}\raisebox{1.5pt}{$67(1.3\text{x})$}} &\progressbar[linecolor=white, filledcolor=black, ticksheight=0.0, heighta=8pt, width=2.5em]{0.15}{\hspace{0.1cm}\raisebox{1.5pt}{$15(1.0\text{x})$}}  & \progressbar[linecolor=white, filledcolor=black, ticksheight=0.0, heighta=8pt, width=2.5em]{0.34}{\hspace{-0.05cm}\raisebox{1.5pt}{\hspace{4pt}$34(0.4\text{x})$}} & \progressbar[linecolor=white, filledcolor=black, ticksheight=0.0, heighta=8pt, width=2.5em]{0.8}{\hspace{0.35cm}\raisebox{1.5pt}{\hspace{7.5pt}$80+ (-)$}}  &\progressbar[linecolor=white, filledcolor=black, ticksheight=0.0, heighta=8pt, width=2.5em]{0.8}{\hspace{0.35cm}\raisebox{1.5pt}{\hspace{7.5pt}$80+ (-)$}}  &\progressbar[linecolor=white, filledcolor=black, ticksheight=0.0, heighta=8pt, width=2.5em]{0.24}{\hspace{0.1cm}\raisebox{1.5pt}{$24(3.3\text{x})$}}  & \progressbar[linecolor=white, filledcolor=black, ticksheight=0.0, heighta=8pt, width=2.5em]{0.51}{\hspace{0.1cm}\raisebox{1.5pt}{$51(1.1\text{x})$}} \\
Ours & \progressbar[linecolor=white, filledcolor=black, ticksheight=0.0, heighta=8pt, width=2.5em]{0.27}{\hspace{0.1cm}\raisebox{1.5pt}{$\textcolor{red}{\textbf{27(2.0\text{x})}}$}} & \progressbar[linecolor=white, filledcolor=black, ticksheight=0.0, heighta=8pt, width=2.5em]{0.5}{\hspace{0.1cm}\raisebox{1.5pt}{$\textcolor{red}{\textbf{50(1.8\text{x})}}$}} & \progressbar[linecolor=white, filledcolor=black, ticksheight=0.0, heighta=8pt, width=2.5em]{0.09}{\hspace{0.1cm}\raisebox{1.5pt}{\hspace{4pt}$\textcolor{red}{\textbf{9(1.6\text{x})}}$}} &\progressbar[linecolor=white, filledcolor=black, ticksheight=0.0, heighta=8pt, width=2.5em]{0.05}{\hspace{0.1cm}\raisebox{1.5pt}{\hspace{4pt}$\textcolor{red}{\textbf{5(3.0\text{x})}}$}}  & \progressbar[linecolor=white, filledcolor=black, ticksheight=0.0, heighta=8pt, width=2.5em]{0.37}{\hspace{0.1cm}\raisebox{1.5pt}{\hspace{4.3pt}$\textcolor{red}{\mathbf{37(>2.7\text{x})}}$}} &\progressbar[linecolor=white, filledcolor=black, ticksheight=0.0, heighta=8pt, width=2.5em]{0.66}{\hspace{0.1cm}\raisebox{1.5pt}{\hspace{7.5pt}$66(>1.5\text{x})$}}  &\progressbar[linecolor=white, filledcolor=black, ticksheight=0.0, heighta=8pt, width=2.5em]{0.12}{\hspace{0.1cm}\raisebox{1.5pt}{$\textcolor{red}{\textbf{12(6.7\text{x})}}$}}  & \progressbar[linecolor=white, filledcolor=black, ticksheight=0.0, heighta=8pt, width=2.5em]{0.15}{\hspace{0.1cm}\raisebox{1.5pt}{$\textcolor{red}{\textbf{15(3.7\text{x}})}$}} \\ \bottomrule
\end{tabular}
\end{minipage}
}
\end{table*}

We demonstrate the performance of our proposed approach in the FL setup with data heterogeneity in this section. We compare our method with the existing state-of-the-art algorithms on various datasets and deep neural networks. For the baseline approaches, we finetune the hyperparameters and only show the best performance we get. Our main findings are 1) we are more communication efficient than the baseline approaches, 2) conformal prediction is an effective tool to improve FL performance in high data heterogeneity scenarios, and 3) the benefit of the trade-off between diversity and uniformity for using deep neural networks in FL.

\subsection{Communication efficiency and accuracy}
We first report the number of rounds required to achieve a certain level of Top 1\% accuracy ($66\%$ for CIFAR10 and $44\%$ for CIFAR100) in Table.~\ref{tab:communication_round}. An algorithm is more communication efficient if it requires less number of rounds to achieve the same accuracy and/or if it transmits fewer number of parameters between the clients and server. Compared to the baseline approaches, we require much fewer number of rounds for almost all types of data heterogeneity and models. We can achieve a speedup between $1.5$ and $6.7$ than FedAvg. We also observe that ResNet-8 tends to converge slower than VGG-11, which may be due to the aggregation of the Batch Normalization layers that are discrepant between the local data distribution~\cite{DBLP:conf/nips/LinKSJ20}. 

We next compare the top-1 accuracy between centralized learning and federated learning algorithms. For the centralized learning experiment, we tune the learning rate from $\{0.01, 0.05, 0.1\}$ and report the best test accuracy based on the validation dataset. 
We train the model for 800 epochs which is as same as the total number of epochs in the federated learning algorithms (80 communication rounds x 10 local epochs). The results are shown in Table.~\ref{tab:top_1_accuracy}. We also show the number of copies of the parameters that need to be transmitted between the server and clients (\eg 2x means we communicate $\boldsymbol{x}$ and $\boldsymbol{y}_{i}$)

Table.~\ref{tab:top_1_accuracy} shows that our approach achieves a much better Top-1 accuracy compared to FedAvg with transmitting a similar or slightly bigger number of parameters between the server and client per round.  Our method also achieves slightly better accuracy than centralized learning when the data is less heterogeneous (\eg $\alpha=0.5$ for CIFAR10 and $\alpha=1.0$ for CIFAR100).

\begin{table*}[ht!]
\caption{The top-1 accuracy (\%) after running 80 communication rounds using different methods on CIFAR10 and CIFAR100, together with the number of communicated parameters between the client and the server. We train the centralised model for 800 epochs (= 80 rounds x 10 local epochs in FL). Higher accuracy is better, and we highlight the best accuracy in red colour.}
\label{tab:top_1_accuracy}
\vspace*{-3mm}
\resizebox{0.9\textwidth}{!}
{\begin{minipage}{\textwidth}
\begin{tabular}{@{}lcccccccccc@{}}
\toprule
 & \multicolumn{5}{c}{VGG-11} & \multicolumn{5}{c}{ResNet-8} \\ 
 \cmidrule(lr){2-6}\cmidrule(l){7-11}
 & \multicolumn{2}{c}{CIFAR10} & \multicolumn{2}{c}{CIFAR100} & server$\Leftrightarrow$client & \multicolumn{2}{c}{CIFAR10} & \multicolumn{2}{c}{CIFAR100} & server$\Leftrightarrow$client \\ 
 \cmidrule(lr){2-3}\cmidrule(lr){4-5}
 \cmidrule(lr){7-8}\cmidrule(lr){9-10}
 & $\alpha=0.1$ & $\alpha=0.5$ & $\alpha=0.1$ & $\alpha=1.0$ & & $\alpha=0.1$ & $\alpha=0.5$ & $\alpha=0.1$ & $\alpha=1.0$ & \\ 
 \midrule
Centralised & \multicolumn{2}{c}{\textcolor{red}{\textbf{87.5}}} & \multicolumn{2}{c}{56.3} & - & \multicolumn{2}{c}{83.4} & \multicolumn{2}{c}{\textcolor{red}{\textbf{56.8}}} & - \\ 
FedAvg & 69.3 & 80.9  & 34.3 & 45.0 & 2x & 64.9  & 79.1  & 38.8  & 47.0 & 2x\\
Fedprox & 72.1 & 80.4 &  35.0  & 43.2 &2x & 66.1 & 77.9 & 42.0 & 47.2  & 2x \\
SCAFFOLD & 74.1 & 83.5 &  43.4 & 50.6 &4x &66.6  & 80.3 & 43.8 & 52.3  & 4x \\
FedDyn & 77.4 & 80.1 &43.8 &45.2 & 2x & 63.8 &72.9 &36.4 &48.1  & 2x \\
Ours & 78.2 & 84.9 & 43.5 & \textcolor{red}{\textbf{58.0}} & 2.1x & 69.3 & \textcolor{red}{\textbf{83.6}} & 43.5 & 52.3 & 2.02x \\ \bottomrule
\end{tabular}
\end{minipage}}
\end{table*}

\subsection{Conformal prediction}
When the data heterogeneity is high across clients, it is difficult for a federated learning algorithm to match the centralized learning performance~\cite{DBLP:journals/corr/abs-2106-05001}. Therefore, we demonstrate the benefit of using simple post-processing conformal prediction to improve the model performance.

We examine the relationship between the empirical coverage and the average predictive set size for the server model after 80 communication rounds for each federated learning algorithm. The empirical coverage is the percentage of the data samples where the correct prediction is in the predictive set, and the average predictive size is the average of the length of the predictive sets over all the test images~\cite{angelopoulos2021uncertainty}. 
See Appendix.~\textcolor{red}{C} for more information about conformal prediction setup and results.

The results for when $\alpha=0.1$ for both datasets and architectures are shown in Fig.~\ref{fig:conformal_prediction}. We show that by slightly increasing the predictive set size, we can achieve a similar accuracy as the centralized performance. Besides, our approach tends to surpass the centralized top-1 performance similar to or faster than other approaches. 
In sensitive use cases such as chemical threat detection, conformal prediction is a valuable tool to achieve certified accuracy at the cost of a slightly larger predictive set size.

\begin{figure}[ht!]
    \centering
    \includegraphics[width=.48\textwidth]{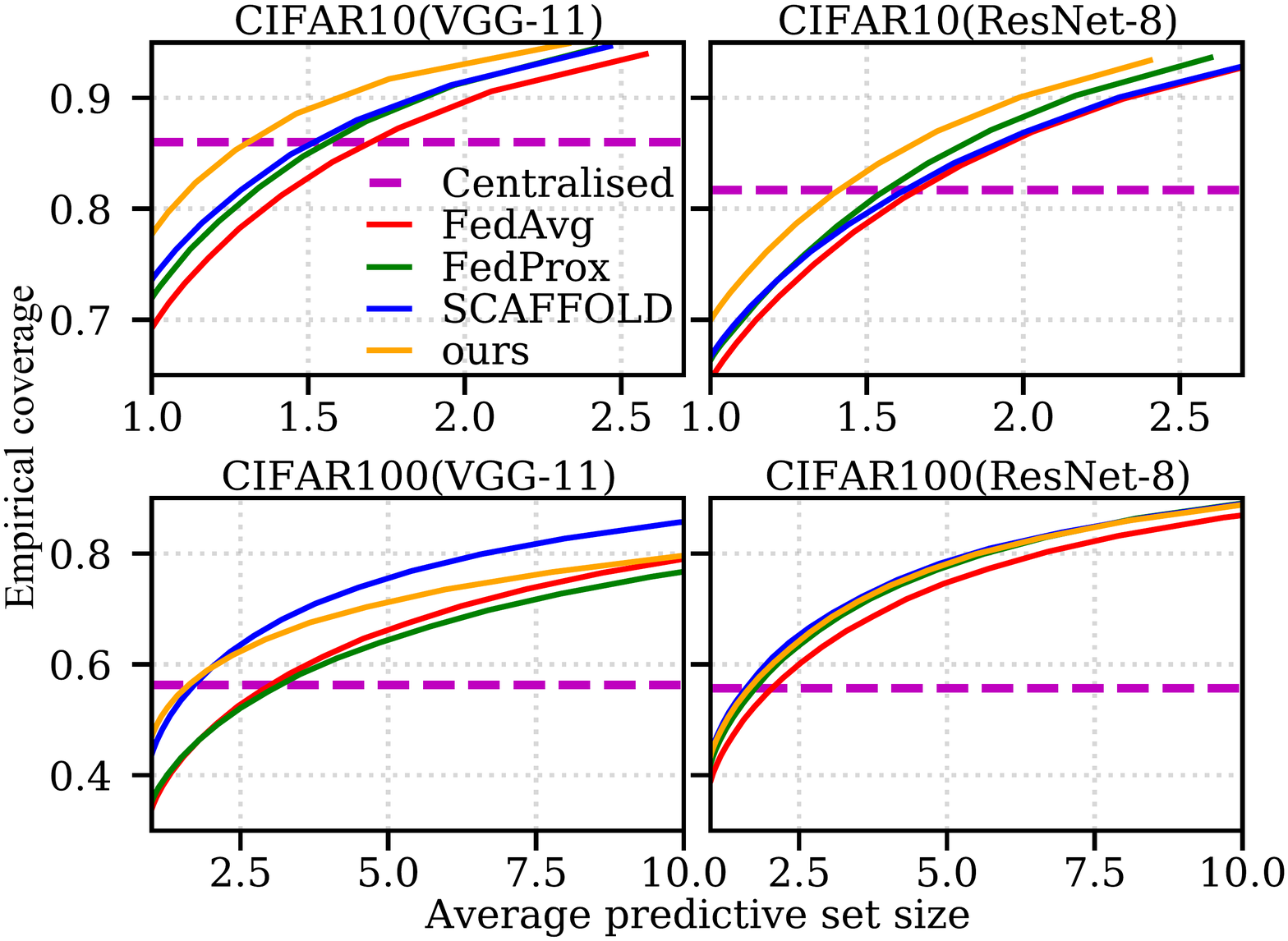}
    \caption{Relation between average predictive size and empirical coverage when $\alpha=0.1$. By slightly increasing the predictive set size, we can achieve a similar performance as the centralised model (Top-1 accuracy) even if the data are heterogeneously distributed across clients. Our method is similar to or faster than other approaches to surpass the centralised Top-1 accuracy. } 
    \label{fig:conformal_prediction}
\end{figure}
\vspace*{-4mm}

\subsection{Diversity and uniformity}

\begin{figure}[ht!]
    \centering
    \includegraphics{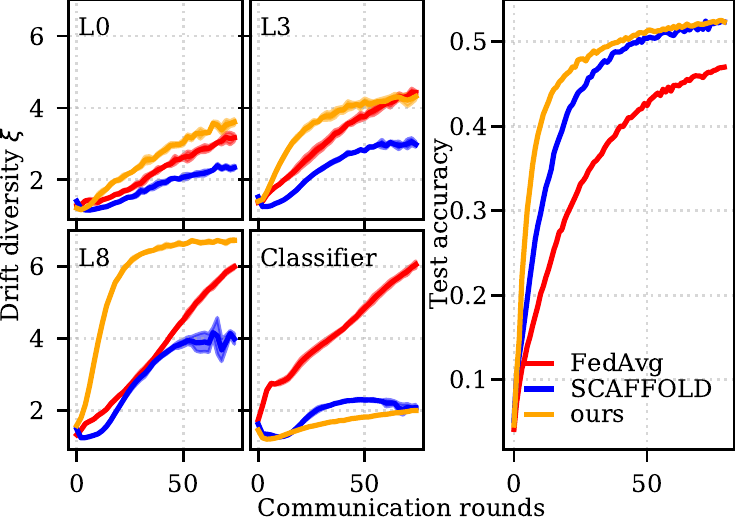}
    \caption{Drift diversity and learning curve for ResNet-8 on CIFAR100 with $\alpha=1.0$. Compared to FedAvg, SCAFFOLD and our method can both improve the agreement between the classifiers. Compared to SCAFFOLD, our method results in a higher gradient diversity at the early stage of the communication, which tends to boost the learning speed as the curvature of the drift diversity seem to match the learning curve.}
    \label{fig:diversity}
\end{figure}

We have shown that our algorithm achieves a better speedup and performance against the existing approaches with only lightweight modifications to FedAvg. We next investigate what factors lead to better accuracy. Specifically, we calculate the drift diversity  $\xi$ across clients after each communication round using Eq.~\ref{eq:drift_diversity} and average $\xi$ across three runs. We show the result of using ResNet-8 and CIFAR100 with $\alpha=1.0$ in Fig.~\ref{fig:diversity}.

Fig.~\ref{fig:diversity} shows the drift diversity for different layers in ResNet-8 and the testing accuracy along the communication rounds. We observe that classifiers have the highest diversity in FedAvg against other layers and methods. SCAFFOLD, which applies control variate on the entire model, can effectively reduce the disagreement of the directions and scales of the averaged gradient across clients. Our proposed algorithm that performs variance reduction only on the classifiers can reduce the diversity of the classifiers even further but increase the diversity of the feature extraction layers. This high diversity tends to boost the learning speed as the curvature of the diversity movement (Fig.~\ref{fig:diversity} left) seems to match the learning curve (Fig.~\ref{fig:diversity} right). Based on this observation, we hypothesize that this diversity along the feature extractor and the uniformity of the classifier is the main reason for our better speedup.

To test this hypothesis, we perform an experiment where we use variance reduction starting from different layers of a neural network. If the starting position of the use of variance reduction influences the learning speed, it indicates where in a neural network we need more diversity and where we need more uniformity. We here show the result of using VGG-11 on CIFAR100 with $\alpha=1.0$ as there are more layers in VGG-11. The result is shown in Fig.~\ref{fig:influence_of_svr} where $\texttt{SVR:}16\rightarrow20$ is corresponding to our approach and $\texttt{SVR:}0\rightarrow20$ is corresponding to SCAFFOLD that applies variance reduction for the entire model. Results for using ResNet-8 is shown in Appendix.~\textcolor{red}{C}.

\begin{figure}[ht!]
    \centering
    \includegraphics{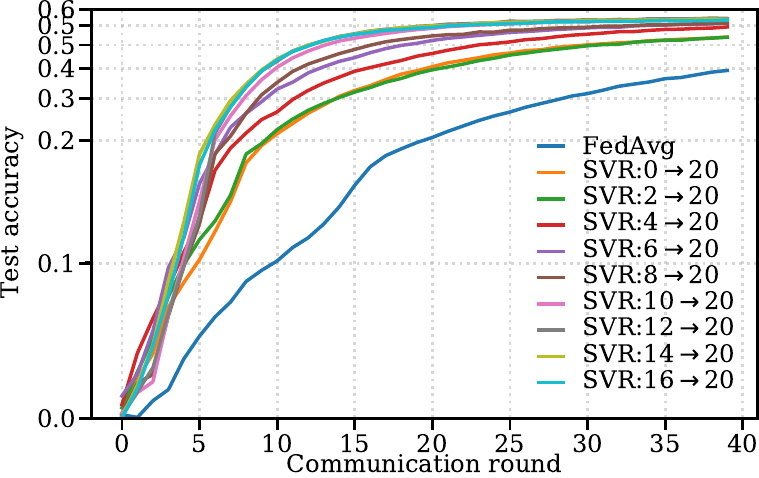}
    \caption{Influence of using stochastic variance reduction(SVR) on layers that start from different positions in a neural network on the learning speed. \texttt{SVR:0$\rightarrow$20} applies variance reduction on the entire model (SCAFFOLD). \texttt{SVR:16$\rightarrow$20} applies variance reduction from the layer index 16 to 20 (ours). The later we apply variance reduction, the better performance speedup we obtain. However, no variance reduction (FedAvg) performs the worst here.}
    \label{fig:influence_of_svr}
\end{figure}

We see from Fig.~\ref{fig:influence_of_svr} that the deeper in a neural network we apply variance reduction, the better learning speedup we can obtain. There is no clear performance difference between where to activate the variance reduction when the layer index is over 10. However, applying no variance reduction (FedAvg) achieves by far the worst performance. We believe that these experimental results indicate that in a distributed optimization framework, to boost the learning speed of an over-parameterized model, we need some levels of diversity in the middle and early layers for learning richer feature representation and some degrees of uniformity in the classifiers for making a less biased decision. 

\section{Conclusion}
\label{sec:conclusion}

In this work, we studied stochastic gradient descent learning for deep neural network classifiers in a federated learning setting, where each client updates its local model using stochastic gradient descent on local data. A central model is periodically updated (by averaging local model parameters) and broadcast to the clients under a communication bandwidth constraint. When data is homogeneous across clients, this procedure is comparable to centralized learning in terms of efficiency; however, when data is heterogeneous, learning is impeded. Our hypothesis for the primary reason for this is that when the local models are out of alignment, updating the central model by averaging is ineffective and sometimes even destructive. 
 
Examining the diversity across clients of their local model updates and their learned feature representations, we found that the misalignment between models is much stronger in the last few neural network layers than in the rest of the network. This finding inspired us to experiment with aligning the local models using a partial variance reduction technique applied only on the last layers, which we named FedPVR. We found that this led to a substantial improvement in convergence speed compared to the competing federated learning methods. In some cases, our method even outperformed centralized learning. We derived a bound on the convergence rate of our proposed method, which matches the rates for SGD when the gradient diversity across clients is sufficiently low. Compared with FedAvg, the communication cost of our method is only marginally worse, as it requires transmitting control variates for the last layers.
 
We believe our FedPVR algorithm strikes a good balance between simplicity and efficiency, requiring only a minor modification to the established FedAvg method; however, in our further research, we plan to pursue more optimal methods for aligning and guiding the local learning algorithms, \eg using adaptive procedures. Furthermore, the degree of over-parameterization in the neural network layers (\eg feature extraction vs bottlenecks) may also play an important role, which we would like to understand better.

\section*{Acknowledgements}
The first three authors thank for financial support from the European Union’s Horizon 2020 research and innovation programme under grant agreement no. 883390 (H2020-SU-SECU-2019 SERSing Project). BL thanks for the financial support from the Otto Mønsted Foundation. 
\newpage

{\small
\bibliographystyle{ieee_fullname}
\bibliography{egbib2}
}

\newpage
\onecolumn

\appendix
\numberwithin{equation}{section}
\numberwithin{figure}{section}
\numberwithin{table}{section}

\newpage
{\hspace{-0.5cm}\Huge\textbf{Appendix}}
\vspace{0.5cm}
\section{Technicalities}
We first summarize the assumptions that are needed for the proof of convergence in the Section~\ref{sec:assumption} based on the previous literature~\cite{DBLP:journals/corr/abs-2003-10422, DBLP:journals/corr/abs-1910-06378}. We then demonstrate the implications of these assumptions for our proof in the Section~\ref{sec:implication}. Following that, we summarize some of the useful and well-known lemmas in the Section~\ref{sec:lemma}. 

\subsection{Assumptions}
\label{sec:assumption}
\noindent\textbf{Assumptions on the objective function}

For some of our results we assume (strong) convexity.
\begin{assumptionp}{A-1}[$\mu$-convex]
\label{assum:mu_convex}
$f_i$ is $\mu$-convex for $\mu \geq 0$ and satisfies:
\begin{equation}
    \langle \nabla f_i(\boldsymbol{x}), \boldsymbol{y}-\boldsymbol{x}\rangle \leq -(f_i(\boldsymbol{x}) - f_i(\boldsymbol{y}) + \frac{\mu}{2}||\boldsymbol{x}-\boldsymbol{y}||^2), \quad \forall \boldsymbol{x}, \boldsymbol{y}\in \mathbb{R}^d, i\in [N] \,.
\end{equation}
Here, we allow $\mu=0$ (we refer generate convex case as when $\mu = 0$)
\end{assumptionp}

For all our theoretical analysis, we assume $f_i$ is smooth.

\begin{assumptionp}{A-2}[$\beta$-smooth]
\label{assum:beta_smooth}
$f_i$ is $\beta$-smooth and satisfy:
\begin{equation}
    ||\nabla f_i(\boldsymbol{x}) - \nabla f_i(\boldsymbol{y})|| \leq \beta||\boldsymbol{x}-\boldsymbol{y}||, \quad \forall \boldsymbol{x}, \boldsymbol{y}\in \mathbb{R}^d, i\in [N]  ,.
\end{equation}
\begin{equation}
    f_i(\boldsymbol{y}) \leq f_i(\boldsymbol{x}) + \langle \nabla f_i(\boldsymbol{x}), \boldsymbol{y}-\boldsymbol{x}\rangle + \frac{\beta}{2}||\boldsymbol{y}-\boldsymbol{x}||^2, \quad \forall \boldsymbol{x}, \boldsymbol{y}\in \mathbb{R}^d, i\in [N]  \,.
\end{equation}
\end{assumptionp}

\begin{remark}
If functions $\{f_i\}$ are convex and $\boldsymbol{x}^*$ is a minimizer of $f$, then $\sum_{i}\nabla f_i(\boldsymbol{x}^*) = 0$, Assumption.~\ref{assum:beta_smooth} implies:
\begin{equation}
    \frac{1}{N}\sum_{i=1}^N ||\nabla f_i(\boldsymbol{x}) - \nabla f_i(\boldsymbol{x}^*)||^2 \leq 2\beta(f(\boldsymbol{x}) - f^*) \quad \forall \boldsymbol{x} \in \mathbb{R}^d \,.
\end{equation}
\end{remark}

\noindent\textbf{Assumptions on the noise}

For the convergence analysis of SGD on convex functions, it is usually enough to assume a bound on the noise on the optimum only(~\cite{DBLP:journals/corr/abs-2003-10422, DBLP:journals/corr/abs-1907-04232}). Similarity, to express the function heterogeneity at the optimum point $\boldsymbol{x}^*$ (such a point always exist for the convex function), we make the following assumption.

\begin{assumptionp}{A-3}[$\zeta$-heterogeneity]
\label{assum:zeta}
We define a measure of variance at the optimum $\boldsymbol{x}^*$ given $N$ clients as :
\begin{equation}
\zeta^2 := \frac{1}{N}\sum_{i=1}^{N}||\nabla f_i(\boldsymbol{x}^*)||^2 \,.
\end{equation}
\end{assumptionp}

For the non-convex function, such an unique optimum point $\boldsymbol{x}^*$ does not necessarily exist, so we generalise assumption~\ref{assum:zeta} to:
\begin{assumptionp}{A-4}[$\hat{\zeta}$-heterogeneity]
\label{assum:zeta_hat}
We assume that there exists constants $\hat{\zeta}$ such that $\forall \boldsymbol{x} \in \mathbb{R}^d$:
\begin{equation}
    \frac{1}{N}\sum_{i=1}^{N}||\nabla f_i(\boldsymbol{x})||^2 \leq \hat{\zeta}^2 \,.
\end{equation}
\end{assumptionp}

Another assumption that is usually common is to assume the stochastic gradients are bounded as:
\begin{assumptionp}{A-5}[bounded variance]
\label{assum:sigma_bound}
$g_i(\boldsymbol{x}):= \nabla f_i(\boldsymbol{x};\mathcal{D}_i)$ is unbiased stochastic gradient of $f_i$ with bounded variance:
\begin{equation}
    \mathbb{E}_{\mathcal{D}_i}[||g_i(\boldsymbol{x}) - \nabla f_i(\boldsymbol{x})||^2] \leq \sigma^2, \quad \forall x\in \mathbb{R}^d, i\in [N] \,.
\end{equation}
\end{assumptionp}

\subsection{Implications of the assumptions}
\label{sec:implication}

Given a binary mask $\boldsymbol{p}$ that has the same length as $\boldsymbol{x}$, it holds that:
\begin{equation}
    ||\boldsymbol{p}\odot \boldsymbol{x}|| \leq ||\boldsymbol{x}||\,.
    \label{eq:p_rule}
\end{equation}
Based on Eq.~\ref{eq:p_rule}, we have the following propositions:

\begin{proposition}[Implications of the smoothness Assumption~\ref{assum:beta_smooth}]
Given a binary mask $\boldsymbol{p}$, we define the block of weights that are variance reduced as $\beta_p$-smooth:
\begin{equation}
    ||\boldsymbol{p}\odot(\nabla f_i(\boldsymbol{x}) - f_i(\boldsymbol{y}))|| \leq \beta_p||\boldsymbol{x}-\boldsymbol{y}||, \quad \forall \boldsymbol{x}, \boldsymbol{y}\in \mathbb{R}^d, i\in [N] \,.
\end{equation}

If Assumption.~\ref{assum:beta_smooth} holds, then it also holds that:
\begin{equation}
    \beta_p \leq \beta \,.
\end{equation}
\end{proposition}

\begin{proposition}[Implication of the \textbf{convex} function heterogeneity Assumption~\ref{assum:zeta}]
Given a binary mask $\boldsymbol{p}$, we define the heterogeneity of the block of weights that are not variance reduced at the optimum $\boldsymbol{x}^*$ as:
\begin{equation}
    \zeta_{1-p}^2 := \frac{1}{N}\sum_{i=1}^{N}||(\boldsymbol{1}-\boldsymbol{p})\odot\nabla f_i(\boldsymbol{x^*})||^2 , 
\end{equation}
If Assumption~\ref{assum:zeta} holds, then it also holds that:
\begin{equation}
    \zeta_{1-p}^2 \leq \zeta^2\,.
\end{equation}
\end{proposition}

\begin{proposition}[Implication of the \textbf{non-convex} function heterogeneity Assumption~\ref{assum:zeta_hat}]
Given a binary mask $\boldsymbol{p}$, we assume the heterogeneity of the block of weights that are not variance reduced as:
\begin{equation}
    \frac{1}{N}\sum_{i=1}^{N}||(\boldsymbol{1}-\boldsymbol{p})\odot \nabla f_i(\boldsymbol{x})||^2  \leq \hat{\zeta}_{1-p} , 
\end{equation}
If Assumption~\ref{assum:zeta_hat} holds, then it also holds that:
\begin{equation}
    \hat{\zeta}_{1-p}^2 \leq \hat{\zeta}^2\,.
\end{equation}
\end{proposition}

\begin{proposition}[Implication of the bounded variance Assumption.~\ref{assum:sigma_bound}]
If Assumption.~\ref{assum:sigma_bound} holds, then:
\begin{subequations}
\begin{equation}
    \mathbb{E}_{\mathcal{D}_i}[||\boldsymbol{p}\odot(g_i(\boldsymbol{x}) - \nabla f_i(\boldsymbol{x}))||^2] \leq \sigma_p^2, \quad \forall \boldsymbol{x} \in \mathbb{R}^d, i\in [N]
\end{equation}
\begin{equation}
    \sigma_p^2\leq\sigma^2  \,.
\end{equation}
\end{subequations}
\end{proposition}

\subsection{Some technical lemmas}
\label{sec:lemma}
We summarize some of the well-known lemmas in this subsection.

\begin{lemma}[triangle inequality]
\label{lemma:triangle}
For arbitrary set of $n$ vectors $\{\boldsymbol{v}_i\}_{i=1}^N$ with $\boldsymbol{v}_i \in \mathbb{R}^d$, then the following are true:
\begin{equation}
||\sum_{i=1}^{N}\boldsymbol{v}_i||^2 \leq N\sum_{i=1}^N||\boldsymbol{v}_i||^2 \,.   
\end{equation}
\begin{equation}
    ||\boldsymbol{v}_i + \boldsymbol{v}_j||^2 \leq (1+\alpha)||\boldsymbol{v}_i||^2 + (1+\alpha^{-1})||\boldsymbol{v}_j||^2, \quad \forall \alpha > 0 \,.
    \label{eq:triangle_2}
\end{equation}
\end{lemma}

\begin{lemma}[separating the mean and variance, {\cite[Lemma.4]{DBLP:journals/corr/abs-1910-06378}}]
\label{lemma:separate_mean_var}

Let $\{\boldsymbol{v}_1, ..., \boldsymbol{v}_\tau\}$ be $\tau$ random variables in $\mathbb{R}^d$ which are not necessarily dependent. First, suppose that their mean is $\mathbb{E}[\boldsymbol{v}_i]=\boldsymbol{\Xi}_i$ and the variance is bounded as $\mathbb{E}||\boldsymbol{v}_i - \boldsymbol{\Xi}_i||^2 \leq \sigma^2$, then the following holds:
\begin{equation}
    \mathbb{E}||\sum_{i=1}^\tau \boldsymbol{v}_i||^2 \leq ||\sum_{i=1}^\tau \boldsymbol{\Xi}_i||^2 + \tau^2\sigma^2 \,.
\end{equation}

Now, suppose their condition mean is $\mathbb{E}[\boldsymbol{v}_i|\boldsymbol{v}_{i-1}, \boldsymbol{v}_{i-1}, .., v_1] = \boldsymbol{\Xi}_i$, the variance is bounded as $\mathbb{E}||\boldsymbol{v}_i-\boldsymbol{\Xi}_i||^2 \leq \sigma^2$, then we can show tighter bounds:
\begin{equation}
    \mathbb{E}||\sum_{i=1}^\tau \boldsymbol{v}_i||^2 \leq 2||\sum_{i=1}^\tau \boldsymbol{\Xi}_i||^2 + 2\tau\sigma^2\,.
\end{equation}
\begin{proof}
For any random variables $X$, $\mathbb{E}[X^2] = (\mathbb{E}[X-\mathbb{E}[X]])^2 + (\mathbb{E}[X])^2$, this implies:
\begin{equation*}
    \mathbb{E}||\sum_{i=1}^\tau \boldsymbol{v}_i||^2 \leq \mathbb{E}||\sum_{i=1}^\tau \boldsymbol{v}_i - \boldsymbol{\Xi}_i||^2 + \mathbb{E}||\sum_{i=1}^\tau \boldsymbol{\Xi}_i||^2 \leq \tau^2\sigma^2 + ||\sum_{i=1}^\tau \boldsymbol{\Xi}_i||^2
\end{equation*}

For the second statement, $\boldsymbol{\Xi}_i$ is not determinant and is dependent on $[v_{i-1}, v_{i-2}, ..., v_1]$. Based on Lemma~\ref{lemma:triangle}:
\begin{subequations}
\begin{equation*}
    \mathbb{E}||\sum_{i=1}^\tau \boldsymbol{v}_i||^2 \leq 2\mathbb{E}||\sum_{i=1}^\tau \boldsymbol{v}_i - \boldsymbol{\Xi}_i||^2 + 2\mathbb{E}||\sum_{i=1}^\tau\boldsymbol{\Xi}_i||^2
\end{equation*}
\begin{equation*}
    \mathbb{E}||\sum_{i=1}^\tau \boldsymbol{v}_i-\boldsymbol{\Xi}_i||^2 = \sum_{i=1}^\tau \mathbb{E}||\boldsymbol{v}_i-\boldsymbol{\Xi}_i||^2 +2\sum_{i,j}\mathbb{E}\langle \boldsymbol{v}_i-\boldsymbol{\Xi}_i, \boldsymbol{v}_j - \boldsymbol{\Xi}_j \rangle \leq \tau \sigma^2
\end{equation*}
\end{subequations}
\end{proof}
\end{lemma}

\begin{lemma}[contrastive mapping, {\cite[Lemma.6]{DBLP:journals/corr/abs-1910-06378}}]
\label{lemma:contrastive_mapping}
For any $\beta$-smooth and $\mu$-strongly convex function $h$, points $\boldsymbol{x}$ and $\boldsymbol{y}$ in the domain of h, and step-size $\eta \leq \frac{1}{\beta}$, the following is true:
\begin{equation}
    ||\boldsymbol{x} - \eta \nabla h(\boldsymbol{x}) + \eta \nabla h(\boldsymbol{y}) + \boldsymbol{y}|| \leq (1-\mu\eta)||\boldsymbol{x}-\boldsymbol{y}||^2 \,.
    \label{eq:contrastive_mapping} 
\end{equation}
\begin{proof}
\begin{equation*}
    \begin{split}
        ||\boldsymbol{x} - \eta\nabla h(\boldsymbol{x}) - y + \eta\nabla h(\boldsymbol{y})||^2 &= ||\boldsymbol{x} - \boldsymbol{y}||^2 + \eta^2||\nabla h(\boldsymbol{x}) - \nabla h(\boldsymbol{y})||^2 - 2\eta \langle \nabla h(\boldsymbol{x}) - \nabla h(\boldsymbol{y}), \boldsymbol{x}-\boldsymbol{y}\rangle \\
        &\leq ||\boldsymbol{x}-\boldsymbol{y}||^2 + (\eta^2\beta-2\eta)\langle \nabla h(\boldsymbol{x}) - \nabla h(\boldsymbol{y}), \boldsymbol{x}-\boldsymbol{y}\rangle\\
        &\leq (1-\eta\mu)||\boldsymbol{x}-\boldsymbol{y}||^2 \,.
    \end{split}
\end{equation*}
The second step uses the smoothness Assumption.~\ref{assum:beta_smooth}. The last step uses our bound on the step size $\eta \leq \frac{1}{\beta}$ (implies $\eta^2\beta - 2\eta \leq -\eta$)
\end{proof}
\end{lemma}

\begin{lemma}[Perturbed strong convexity, {\cite[Lemma.5]{DBLP:journals/corr/abs-1910-06378}}]
\label{lemma:perturbed_strong_convexity}
The following holds for any $\beta$-smooth and $\mu$-strongly convex function $h$ and any $\boldsymbol{x}, \boldsymbol{y}, \boldsymbol{z}$ in the domain of $h$:
\begin{equation}
    \langle \nabla h(\boldsymbol{x}), \boldsymbol{z}-\boldsymbol{y}\rangle \geq h(\boldsymbol{z}) - h(\boldsymbol{y}) + \frac{\mu}{4}||\boldsymbol{y}-\boldsymbol{z}||^2 - \beta||\boldsymbol{z}-\boldsymbol{y}||^2\,.
\end{equation}
\begin{proof}
Given any $\boldsymbol{x}, \boldsymbol{y}$, and $\boldsymbol{z}$, we can get the following inequalities using smoothness and strong convexity of $h$:
\begin{subequations}
\begin{equation*}
    \langle \nabla h(\boldsymbol{x}), \boldsymbol{z}-\boldsymbol{x} \rangle \geq h(\boldsymbol{z}) - h(\boldsymbol{x}) - \frac{\beta}{2}||\boldsymbol{z}-\boldsymbol{x}||^2,
\end{equation*}
\begin{equation*}
    \langle \nabla h(\boldsymbol{x}), \boldsymbol{x}-\boldsymbol{y} \rangle \geq h(\boldsymbol{x}) - h(\boldsymbol{y}) - \frac{\mu}{2}||\boldsymbol{x} - \boldsymbol{y}||^2 \,.
\end{equation*}
\end{subequations}
Applying relaxed triangle inequality:
\begin{equation*}
    \frac{\mu}{2}||\boldsymbol{y}-\boldsymbol{x}||^2 \geq \frac{\mu}{4}||\boldsymbol{y}-\boldsymbol{z}||^2 - \frac{\mu}{2}||\boldsymbol{x}-\boldsymbol{z}||^2\,.
\end{equation*}
Combining the above three equations together:
\begin{equation*}
    \langle \nabla h(\boldsymbol{x}), \boldsymbol{z}-\boldsymbol{y}\rangle \geq h(\boldsymbol{z}) - h(\boldsymbol{y})+\frac{\mu}{4}||\boldsymbol{y}-\boldsymbol{z}||^2 - \frac{\beta+\mu}{2}||\boldsymbol{z}-\boldsymbol{x}||^2 \,.
\end{equation*}
The lemma follows since $\beta \geq \mu$
\end{proof}
\end{lemma}

\begin{lemma}[Tunning the stepsize, {\cite[Lemma.17]{DBLP:journals/corr/abs-2003-10422}}]
\label{lemma:tune_stepsize}
For any parameters $r_0\geq 0, b\geq 0, e\geq 0, \gamma\geq 0$, there exists constant step size $\eta \leq \frac{1}{\gamma}$ such that:
\begin{equation}
    \Psi_T := \frac{r_0}{\eta(T+1)} + b\eta + e\eta^2 \leq 2\left(\frac{br_0}{T+1}\right)^{\frac{1}{2}} + 2e^{1/3}\left(\frac{r_0}{T+1}\right)^{\frac{2}{3}} + \frac{\gamma r_0}{T+1}\,.
\end{equation}
\begin{proof}
Choosing $\eta = \text{min}\left\{\left(\frac{r_0}{b(T+1)}\right)^{\frac{1}{2}}, \left(\frac{r_0}{e(T+1)} \right)^{\frac{1}{3}}, \frac{1}{\gamma} \right\} \leq \frac{1}{\gamma}$, we have three cases:
\begin{itemize}
    \item $\eta=\frac{1}{\gamma}$ and is smaller than the other two terms, then
    \begin{equation*}
        \Psi_T \leq \frac{dr_0}{T+1}+\frac{b}{\gamma}+\frac{e}{d^2} \leq \left(\frac{br_0}{T+1} \right)^{\frac{1}{2}} + \frac{dr_0}{T+1} + e^{1/3}\left(\frac{r_0}{T+1}\right)^{\frac{2}{3}},
    \end{equation*}
    \item $\eta = \left(\frac{r_0}{b(T+1)}\right)^{\frac{1}{2}} < \left(\frac{r_0}{e(T+1)} \right)^{\frac{1}{3}}$, then:
    \begin{equation*}
        \Psi_T \leq 2\left(\frac{br_0}{T+1} \right)^{\frac{1}{2}} + e\left(\frac{r_0}{b(T+1)}\right) \leq 2\left(\frac{br_0}{T+1} \right)^{\frac{1}{2}} + e^{1/3}\left(\frac{r_0}{T+1}\right)^{\frac{2}{3}},
    \end{equation*}
    \item The last case $\eta=\left(\frac{r_0}{e(T+1)} \right)^{\frac{1}{3}} < \left(\frac{r_0}{b(T+1)}\right)^{\frac{1}{2}}$, then
    \begin{equation*}
        \Psi_T \leq 2e^{1/3}\left(\frac{r_0}{T+1} \right) + b\left(\frac{r_0}{e(T+1)}\right)^{\frac{1}{3}} \leq 2e^{1/3}\left(\frac{r_0}{T+1}\right)^{\frac{2}{3}} + \left(\frac{br_0}{T+1} \right)^{\frac{1}{2}} \,.
    \end{equation*}
\end{itemize}

\end{proof}

\end{lemma}

\begin{lemma}[tunning the stepsize, {\cite[Lemma.2]{DBLP:journals/corr/abs-1907-04232}}]
\label{lemma:stich_stepsize}
If there exists two non-negative sequences $\{r_t\}_{t\geq0}$, $\{s_t\}_{t\geq0}$ and $a>0$ that satisfy the relation:
\begin{equation}
    r_{t+1} \leq (1-a\mu_t)r_t - b\mu_ts_t + c\mu_t^2,
\end{equation}
Then there exists a constant step size $\mu_t \equiv \mu \leq \frac{1}{\gamma}$ such that for weight $w_t:=(1-a\mu)^{-t-1}$ and $W_T:=\sum_{t=0}^T w_t$, it holds:
\begin{equation}
    \frac{b}{W_t}\sum_{t=0}^T s_tw_t + ar_{T+1} = \tilde{\mathcal{O}}\left(\gamma r_0\exp\left[-\frac{aT}{b}\right] + \frac{c}{aT} \right) \,.
\end{equation}
\begin{proof}
Choosing $\eta = \text{min}\left\{\frac{\ln{(\max\{2, a^2r_0T^2/c\})}}{aT}, \frac{1}{\gamma} \right\}$, we have two cases:
\begin{itemize}
    \item If $\frac{1}{\gamma} \geq \frac{\ln{(\max\{2, a^2r_0T^2/c\})}}{aT}$, then we choose $\eta=\frac{\ln{(\max\{2, a^2r_0T^2/c\})}}{aT}$:
    \begin{equation*}
       \frac{b}{W_t}\sum_{t=0}^T s_tw_t + ar_{T+1} =\tilde{\mathcal{O}}\left(\frac{c}{aT}\right) ,
    \end{equation*}
    \item If  $\frac{1}{\gamma} < \frac{\ln{(\max\{2, a^2r_0T^2/c\})}}{aT}$, then we chose $\eta = \frac{1}{\gamma}$:
    \begin{equation*}
        \frac{b}{W_t}\sum_{t=0}^T s_tw_t + ar_{T+1} =\tilde{\mathcal{O}}\left(\gamma r_0\exp\left[-\frac{aT}{\gamma}\right] +\frac{c}{aT} \right) .
    \end{equation*}
\end{itemize}

\end{proof}

\end{lemma}

\newpage
\section{Convergence of \texttt{FedPVR}}
\label{sec:proof}
We first state the convergence theorem, then provide the proof for the convergence rate. 

\begin{remark}
Proving convergence for a randomly picked iterate $\bar{\boldsymbol{x}}^R \in \{\boldsymbol{x}^r\}_{r=0}^{R}$ is equivalent to show the convergence of a (weighted) average of the output criterion, e.g. $\frac{1}{R+1}\sum_{r=0}^{R}\mathbb{E}[f(\boldsymbol{x}^r)] - f^*$ for convex functions{\cite[Remark 17]{DBLP:journals/corr/abs-1909-05350}}. Following this, we assume there exists some weights of $\{w_r\}$ such that:
\begin{equation*}
    \bar{\boldsymbol{x}}^R = \boldsymbol{x}^{r-1} \quad \text{with probability} \quad  \frac{w_{r}}{\sum_{\tau} w_{\tau}} \quad \text{for} \quad  r\in \{1, ..., R+1\}
\end{equation*}
\end{remark}

\begin{theoremp}{II}
\label{theorem:general}
Suppose functions $\{f_i\}$ satisfies Assumption~\ref{assum:beta_smooth} and~\ref{assum:sigma_bound}. Then in each of the cases, there exist weights $\{w_r\}$ and local step size $\eta_l$ such that for any $\eta_g\geq 1$, the output of \texttt{FedPVR}, satisfies:
\begin{itemize}
    \item \textbf{Strongly convex}: $f_i$ satisfies Assumption.~\ref{assum:mu_convex} and~\ref{assum:zeta} for $\mu > 0$, $\eta_l \leq \min\left(\frac{1}{80K\eta_g\beta}, \frac{26}{20\mu K\eta_g}\right)$, $R\geq\max\left(\frac{20}{13}, \frac{160\beta}{\mu}\right)$, then
    \begin{equation}
        \mathbb{E}[f(\bar{\boldsymbol{x}}^R)] - f(\boldsymbol{x}^*) \leq \tilde{\mathcal{O}}\left(\frac{\sigma^2}{\mu NKR}(1+\frac{N}{\eta_g^2}) + \frac{\zeta_{1-p}^2}{\mu R} + \mu D\exp\left(-\min\left\{\frac{13}{20}, \frac{\mu}{160\beta}\right\}R \right) \right) \,.
    \end{equation}
    \vspace*{-4mm}
    \item \textbf{General convex}: $f_i$ satisfies Assumption.~\ref{assum:mu_convex} and~\ref{assum:zeta} for $\mu=0$, $\eta_l \leq \frac{1}{80K\eta_g\beta}$, then:
    \begin{equation}
        \mathbb{E}[f(\bar{\boldsymbol{x}}^R)] - f(\boldsymbol{x}^*) \leq \mathcal{O}\left(\frac{\sigma \sqrt{D}}{\sqrt{RKN}}\sqrt{1+\frac{N}{\eta_g^2}} + \frac{\zeta_{1-p}\sqrt{D}}{\sqrt{R}} + \frac{\beta D}{R} + F\right)  \,.  
    \end{equation}
    \item \textbf{Non convex}: $f_i$ satisfies Assumption.~\ref{assum:zeta_hat}, $\eta_l \leq \frac{1}{26K\eta_g\beta}$, and $R\geq 1$, then:
    \begin{equation}
        \mathbb{E}||\nabla f(\bar{\boldsymbol{x}}^R)||^2\leq \mathcal{O}\left(\frac{\sigma\sqrt{F}}{\sqrt{KNR}}\sqrt{\beta(\frac{N}{\eta_g^2} + 1)} + \frac{\hat{\zeta}_{1-p}\sqrt{F}}{\sqrt{R}}\sqrt{\frac{\beta}{\eta_g^2}} + \frac{\beta F}{R} \right)
    \end{equation}
\end{itemize}
\vspace*{-4mm}
Here $D:= ||\boldsymbol{x}^0 - \boldsymbol{x}^*||^2$ and $F:= f(\boldsymbol{x}^0) - f(\boldsymbol{x}^*)$
\end{theoremp}

Instead of showing the convergence dependent on the number of round $R$, we can also state the convergence dependent on the expected error $\epsilon$ with choosing $\eta_g = \sqrt{N}$. 

\begin{corollaryp}{I}
\label{theorem_convergence_appendix}
Suppose function $\{f_i\}$ satisfy Assumption~\ref{assum:beta_smooth} and Assumption~\ref{assum:sigma_bound}. Then the output of \texttt{FedPVR} has expected error smaller than $\epsilon$ for $\eta_g = \sqrt{N}$ and some values of $\eta_l$, $R$ satisfying:
\begin{itemize}
    \item \textbf{Strongly convex:} $f_i$ satisfies Assumption.~\ref{assum:mu_convex} and~\ref{assum:zeta} for $\mu > 0$, $\eta_l \leq \min\left(\frac{1}{80K\eta_g\beta}, \frac{26}{20\mu K\eta_g}\right)$
    \begin{equation}
        R = \tilde{\mathcal{O}}\left(\frac{\sigma^2}{\mu NK\epsilon} + \frac{\zeta_{1-p}^2}{\mu\epsilon} + \frac{\beta}{\mu}\right),
    \end{equation}
    \vspace*{-4mm}
    \item \textbf{General convex:} ($\mu=0$):  $f_i$ satisfies Assumption.~\ref{assum:mu_convex} and~\ref{assum:zeta} for $\mu=0$, $\eta_l \leq \frac{1}{80K\eta_g\beta}$,
    \begin{equation}
        R = \mathcal{O}\left(\frac{\sigma^2D}{KN\epsilon^2} + \frac{\zeta_{1-p}^2D}{\epsilon^2} + \frac{\beta D}{\epsilon}  + F\right),
    \end{equation}
    \item \textbf{Non-convex}: $f_i$ satisfies Assumption.~\ref{assum:zeta_hat}, $\eta_l \leq \frac{1}{26K\eta_g\beta}$, and $R\geq 1$, then:
    \begin{equation}
        R=\mathcal{O}\left(\frac{\beta\sigma^2F}{KN\epsilon^2} + \frac{\beta\hat{\zeta}_{1-p}^2F}{N\epsilon^2} + \frac{\beta F}{\epsilon} \right),
    \end{equation}
\end{itemize}
\vspace*{-4mm}
Where $D := ||\boldsymbol{x}^0 - \boldsymbol{x}^*||^2$ and $F:=f(\boldsymbol{x}^0) - f^*$.
\end{corollaryp}

In the special case $\boldsymbol{p}=\boldsymbol{1}$ ($\zeta_{1-p}^2 = 0, \hat{\zeta}_{1-p}^2 = 0$), then FedPVR is identical to Scaffold and we recover their convergence guarantees. In the strongly convex case, the effect of the heterogeneity of the block of weights that are not variance-reduced ($\zeta_{1-p}^2$) becomes negligible if $\tilde{\mathcal{O}}\left(\frac{\zeta_{1-p}^2}{\epsilon}\right)$ is sufficiently smaller than $\tilde{\mathcal{O}}\left(\frac{\sigma^2}{NK\epsilon}\right)$. In such case, our rate becomes $\frac{\sigma^2}{\mu NK\epsilon} + \frac{1}{\mu}$, which recovers the SCAFFOLD in the strongly convex without sampling and further matches the SGD (with mini-batch size $K$ on each worker), proving that it is at least as fast as the SGD.

\begin{remark}[heterogeneity-diversity]
Theoretically, in the non-convex case, the heterogeneity of the block of the weights that are not variance-reduced $\hat{\zeta}_{1-p}:=\frac{1}{N}\sum_{i=1}^N||(\boldsymbol{1}-\boldsymbol{p})\odot \nabla f_i(\boldsymbol{x})||^2$  may slow down the convergence. In the manuscript, we observe that the diversity of the block of weights that are not variance-reduced $\xi^r:=\frac{\sum_{i=1}^N||(\boldsymbol{1}-\boldsymbol{p})\odot(\boldsymbol{y}_{i,K}^r -  \boldsymbol{x}^{r-1})||^2}{||\sum_{i=1}^N (\boldsymbol{1}-\boldsymbol{p})\odot(\boldsymbol{y}_{i,K}^r - \boldsymbol{x}^{r-1})||^2}$ tends to increase the convergence speed. The above two observations do not necessarily disagree as we define $\hat{\zeta}_{1-p}$ and $\xi$ differently. The theory has not yet captured the phenomenon of the diversity of the non-variance-reduced weights improving convergence. 
\end{remark}

\vspace*{-0.5cm}

\subsection{Algorithm and additional definitions}

We write our algorithm using the following notions: $\{\boldsymbol{y}_i\}$ represents the client model, $\boldsymbol{x}$ is the aggregated server model, $\boldsymbol{c_i}$, and $\boldsymbol{c}$ are the client and server control variate. The server maintains a global control variate $\boldsymbol{c}$ and each client maintains its own control variate $\boldsymbol{c_i}$. $N$ is the total number of clients. All the clients participate each round.

To give the proof for the convergence rate and simplify the algorithm, we rewrite it as shown in Algorithm II. The algorithm is identical to Algorithm-1 in the main paper except that $\boldsymbol{c}_i$ and $\boldsymbol{c}$ are now vectors with the same length as the model parameters. However, we only update the block of weights where the corresponding value in $\boldsymbol{p}$ equals 1. This version of the algorithm may consume more bits during the communication between the client and server, but the convergence rate is identical to Algorithm-1 in the main manuscript. 

\begin{algorithm}[ht!]
\floatname{algorithm}{Algorithm II}
\small
\caption{Partial variance reduction (FedPVR)}
\algcomment{This algorithm is identical to Algorithm I in the manuscript in terms of convergence rate but may consume more bits during the communication between the clients and server.}
\label{append:alg}
\hspace*{\algorithmicindent} \textbf{server}: initialise the server model $\boldsymbol{x}$, control variate $\boldsymbol{c}=\boldsymbol{0}$, and global step size $\eta_g$

\hspace*{\algorithmicindent} \textbf{client}: initialise the control variate $\boldsymbol{c}_i=\boldsymbol{0}$, and local step size $\eta_l$ 

\hspace*{\algorithmicindent} \textbf{general}: set a binary mask $\boldsymbol{p} \in \{0, 1\}^d$  

\label{Algorithm:l_scaffold_appendix}
\begin{algorithmic}[1]
\Procedure{Model updating}{}
    \For {$r = 1 \to R $}
    \State \textbf{Communicate} $\boldsymbol{x}$ \textbf{and} $\boldsymbol{c}$ \textbf{to all clients} $i \in [N]$
    \For {\textbf{client $i=1\rightarrow N$ in parallel}}
        \State $\boldsymbol{y}_i \leftarrow \boldsymbol{x}$
        \For {$k = 1 \to K$}
            \State compute minibatch gradient $g_i(\boldsymbol{y}_i)$
            \State $\boldsymbol{y}_i \leftarrow \boldsymbol{y}_{i} - \eta_l(g_i(\boldsymbol{y}_{i}) - \boldsymbol{c_i} + \boldsymbol{c})$
        \EndFor
        \State $\boldsymbol{c}_{i} \leftarrow \boldsymbol{c}_{i} - \boldsymbol{c} + \frac{1}{K\eta_l}\boldsymbol{p}\odot(\boldsymbol{x} - \boldsymbol{y}_{i})$
        \State \textbf{Communicate} $\boldsymbol{y}_{i}, \boldsymbol{c}_i$ \textbf{to the server}
    \EndFor
    \State $\boldsymbol{x} \leftarrow (1-\eta_g)\boldsymbol{x} + \frac{1}{N}\sum_{i=1}^N\boldsymbol{y}_{i}$
    \State $\boldsymbol{c} \leftarrow \frac{1}{N}\sum_{i=1}^N\boldsymbol{c}_{i}$
    \EndFor
\EndProcedure
\Statex
\end{algorithmic}
  \vspace{-0.4cm}%
\end{algorithm}

Given the above algorithm, every client performs the following updates:
\begin{itemize}
    \item Starting from the shared global parameters $\boldsymbol{y}_{i,0}^r = \boldsymbol{x}^{r-1}$, we update the local parameters:
    \begin{equation}
        \boldsymbol{y}_{i, k}^r = \boldsymbol{y}_{i, k-1}^r - \eta_l \boldsymbol{v}_{i,k}^r, \quad \text{where} \quad \boldsymbol{v}_{i,k}^r:=g_i(\boldsymbol{y}_{i,k-1}^r) - \boldsymbol{c}_i^{r-1} + \boldsymbol{c}^{r-1} \,.
    \label{eq:update_0}
    \end{equation}
    
    \item update the control variate
    \vspace*{-0.1cm}
     \begin{equation}
        \boldsymbol{c}_i^r = \boldsymbol{c}_i^{r-1} - \boldsymbol{c}^{r-1} + \frac{1}{K\eta_l}\boldsymbol{p}\odot(\boldsymbol{x}^{r-1} - \boldsymbol{y}_{i, K}^r) \,.
        \label{eq:update_1}
    \end{equation}
    
    \item compute the new global model and global control variate:
    \vspace*{-0.1cm}
    \begin{equation}
        \boldsymbol{x}^{r} = \boldsymbol{x}^{r-1} + \frac{\eta_g}{N}\sum_{i=1}^N(\boldsymbol{y}_{i, K}^r - \boldsymbol{x}^{r-1}) \quad \text{and} \quad \boldsymbol{c}^{r} = \frac{1}{N}\sum_{i=1}^N\boldsymbol{c}_i^r \,.
        \label{eq:update_2}
    \end{equation}
\end{itemize}

\begin{definition}
We define the client drift $\mathcal{E}_r$ to be the amount of the movement between a client model $\boldsymbol{y}_{i,k}^r$ and the starting server model $\boldsymbol{x}^{r-1}$
\begin{equation}
    \mathcal{E}_r := \frac{1}{NK}\sum_{i=1}^N\sum_{k=1}^K\mathbb{E}||\boldsymbol{y}_{i,k}^r - \boldsymbol{x}^{r-1}||^2\,.
\end{equation}
\end{definition}

\begin{definition}
We define $C_r$ as:
\begin{equation}
    C_r := \frac{1}{N}\sum_{i=1}^N\mathbb{E}||\mathbb{E}[\boldsymbol{c}_i^r]- \nabla f_i(\boldsymbol{x}^*)||^2 \,.
\end{equation}
\end{definition}

\begin{definition}
We define the effective learning rate $\tilde{\eta}$ as:
\begin{equation}
    \tilde{\eta} = K\eta_l\eta_g \,.
\end{equation}
\end{definition}

\vspace*{-2mm}

\subsection{Proof for the convergence rate for convex functions}
\begin{lemma}
For updates~\ref{eq:update_0} to~\ref{eq:update_2}, we can bound the variance of the server update in any round $r$ and any $\tilde{\eta}:=\eta_l\eta_gK \geq 0$ as follows
\begin{equation}
\mathbb{E}[||\boldsymbol{x}^r - \boldsymbol{x}^{r-1}||^2] \leq 8\tilde{\eta}^2\beta^2\mathcal{E}_r + 8\tilde{\eta}^2\beta(\mathbb{E}[f(\boldsymbol{x}^{r-1}) - f(\boldsymbol{x}^*)]) + 16\tilde{\eta}^2C_{r-1} + \frac{8\tilde{\eta}^2\sigma^2}{KN} + \frac{16\tilde{\eta}^2\sigma_p^2}{KN} \,.
\end{equation}
\label{lemma1_variance}
\end{lemma}
\vspace*{-8mm}

\begin{proof}
We know that $\boldsymbol{y}_{i,k}^r = \boldsymbol{x}^{r-1} - \eta_l \sum_k\boldsymbol{v}_{i,k}^r$, where $\boldsymbol{v}_{i,k}^r = g_i(\boldsymbol{y}_{i,k-1}^r) - \boldsymbol{c}_i + \boldsymbol{c}$, and $\boldsymbol{x}^r = \boldsymbol{x}^{r-1} + \frac{\eta_g}{N} \sum_i (\boldsymbol{x}^{r-1} - \eta_l \sum_k\boldsymbol{v}_{i,k}^r - \boldsymbol{x}^{r-1})$, so given $\tilde{\eta} = \eta_l\eta_g K$
\begin{equation}
\begin{split}
    \mathbb{E}[||\boldsymbol{x}^r - \boldsymbol{x}^{r-1}||^2] &= \mathbb{E}[||\frac{\eta_g\eta_l}{N}\sum_{i, k}\boldsymbol{v}_{i,k}^r||^2] = \mathbb{E}[||\frac{\tilde{\eta}}{KN}\sum_{i, k} g_i(\boldsymbol{y}_{i,k-1}^r) - \boldsymbol{c}_i^{r-1} + \boldsymbol{c}^{r-1}||^2]\,.
\end{split}
\end{equation}
We drop the round index $r$ and $r-1$ everywhere, we have:
\vspace*{-0.1mm}
\begin{equation}
\begin{split}
    \mathbb{E}[||\Delta \boldsymbol{x}||^2] 
    &\leq 4\tilde{\eta}^2\mathbb{E}||\frac{1}{KN}\sum_{i, k} g_i(\boldsymbol{y}_{i,k-1}) - \nabla f_i(\boldsymbol{x})||^2 + 4\tilde{\eta}^2\mathbb{E}||\frac{1}{N}\sum_i\nabla f_i(\boldsymbol{x}) - \nabla f_i(\boldsymbol{x}^*)||^2 \\
    &\quad + 4\tilde{\eta}^2\mathbb{E}||\frac{1}{N}\sum_i\boldsymbol{c}_i - \nabla f_i(\boldsymbol{x}^*)||^2 + 4\tilde{\eta}^2\mathbb{E}||\boldsymbol{c}||^2 \\
    &\leq 8\tilde{\eta}^2\mathbb{E}||\frac{1}{NK}\sum_{i, k} \nabla f_i(\boldsymbol{y}_{i,k-1}) - \nabla f_i(\boldsymbol{x})||^2 + \frac{4\tilde{\eta}^2}{N}\sum_i\mathbb{E}||\nabla f_i(\boldsymbol{x}) - \nabla f_i(\boldsymbol{x}^*)||^2 \\
    &\quad + 4\tilde{\eta}^2\mathbb{E}||\frac{1}{N}\sum_i\boldsymbol{c}_i-\nabla f_i(\boldsymbol{x}^*)||^2 + 4\tilde{\eta}^2\mathbb{E}||\boldsymbol{c}||^2 + \frac{8\tilde{\eta}^2\sigma^2}{KN}\\
    &\leq \frac{8\tilde{\eta}^2}{NK}\sum_{i, k}\mathbb{E}||\nabla f_i(\boldsymbol{y}_{i,k-1}) - \nabla f_i(\boldsymbol{x})||^2 + \frac{4\tilde{\eta}^2}{N}\sum_{i}\mathbb{E}||\nabla f_i(\boldsymbol{x}) - \nabla f_i(\boldsymbol{x}^*)||^2 \\
    &\quad + 8\tilde{\eta}^2\mathbb{E}||\frac{1}{N}\sum_i\boldsymbol{c}_i - \nabla f_i(\boldsymbol{x}^*)||^2 + \frac{8\tilde{\eta}^2\sigma^2}{KN} \\
    &\leq 8\tilde{\eta}^2\beta^2\mathcal{E}_r + 8\tilde{\eta}^2\beta(\mathbb{E}[f(\boldsymbol{x})] - f(\boldsymbol{x}^*)) + 16\tilde{\eta}^2\frac{1}{N}\sum_i\mathbb{E}||\mathbb{E}[\boldsymbol{c}_i] - \nabla f_i(\boldsymbol{x}^*)||^2 \\
    &\quad + \frac{8\tilde{\eta}^2\sigma^2}{KN} + \frac{16\tilde{\eta}^2\sigma_p^2}{KN} \\
    &\leq 8\tilde{\eta}^2\beta^2\mathcal{E}_r + 8\tilde{\eta}^2\beta(\mathbb{E}[f(\boldsymbol{x}) - f(\boldsymbol{x}^*)]) + 16\tilde{\eta}^2C_{r-1} + \frac{8\tilde{\eta}^2\sigma^2}{KN} + \frac{16\tilde{\eta}^2\sigma_p^2}{KN}\,. 
\end{split}    
\end{equation}

The second step uses Jensen inequality. The third step uses Lemma~\ref{lemma:separate_mean_var} for separating the mean and variance and Lemma~\ref{lemma:triangle}. The variance of $\frac{1}{KN}\sum_{i, k}g_i(\boldsymbol{y}_{i,k})$ is bounded by $\frac{\sigma^2}{KN}$, the variance of $\boldsymbol{c}_i = \frac{1}{K}\sum_k\boldsymbol{p}\odot g_i(\boldsymbol{y}_{i,k})$ is bounded by $\frac{\sigma_p^2}{K}$. The fourth step uses Lemma.~\ref{lemma:triangle} and by definition $\boldsymbol{c} = \frac{1}{N}\sum_i\boldsymbol{c}_i$. the last second step uses the smoothness and convexity of the function. The last step uses the definition of $C_r$ and completes the proof.
\end{proof}

\begin{lemma}
For updates (~\ref{eq:update_0}) - (~\ref{eq:update_2}) with control variate update~\ref{eq:update_1} and Assumption.~\ref{assum:mu_convex} to Assumption.~\ref{assum:zeta}, the following holds true for any $\tilde{\eta} := \eta_l\eta_gK \in [0, \frac{1}{\beta}]$
\begin{equation}
    C_r \leq 4\beta_p^2\mathcal{E}_r + 8\beta_p(\mathbb{E}[f(\boldsymbol{x})] - f(\boldsymbol{x}^*)) + 2\zeta_{1-p}^2
\end{equation}
\label{lemma:cr}
\end{lemma}

\vspace*{-1cm}

\begin{proof}
We first simplify the expression for $\boldsymbol{c}_i$.

\begin{equation}
\begin{split}
    \boldsymbol{c}_i^r &= \boldsymbol{c}_i^{r-1} - \boldsymbol{c}^{r-1} + \frac{1}{K\eta_l}\boldsymbol{p}\odot(\boldsymbol{x}^{r-1} - \boldsymbol{y}_{i,K}^r) \\
    &= \boldsymbol{c}_i^{r-1} - \boldsymbol{c}^{r-1} + \frac{1}{K\eta_l}\boldsymbol{p}\odot(\eta_l\sum_{k}(g_i(\boldsymbol{y}_{i,k}^r) - \boldsymbol{c}_i^{r-1} + \boldsymbol{c}^{r-1})) \\
    &= (\boldsymbol{1}-\boldsymbol{p})\odot(\boldsymbol{c}_i^{r-1} - \boldsymbol{c}^{r-1}) + \frac{1}{K}\sum_{k}\boldsymbol{p}\odot g_i(\boldsymbol{y}_{i,k}^r) \\
    &= \frac{1}{K}\sum_{k}\boldsymbol{p}\odot g_i(\boldsymbol{y}_{i,k-1}^r)\,.
\end{split}
\end{equation}

The zero element in $(\boldsymbol{1} - \boldsymbol{p})$ is corresponding non-zero element in $\boldsymbol{c}_i$, so the element-wise multiplication results in zero. Taking the expectation on both side, we obtain: $\mathbb{E}[\boldsymbol{c}_i] = \frac{1}{K}\sum_{k}\nabla f_i(\boldsymbol{y}_{i,k})$

\vspace*{-4mm}
\begin{equation}
    \begin{split}
        C_r &= \frac{1}{N}\sum_{i}\mathbb{E}||\frac{1}{K}\sum_{k}\boldsymbol{p}\odot\nabla f_i(\boldsymbol{y}_{i,k}) - \nabla f_i(\boldsymbol{x}^*)||^2 \\
        &\leq \frac{2}{NK}\sum_{i, k}\mathbb{E}||\boldsymbol{p}\odot(\nabla f_i(\boldsymbol{y}_{i,k}) - \nabla f_i(\boldsymbol{x}^*))||^2 + \frac{2}{N}\sum_{i}\mathbb{E}||(\boldsymbol{1} - \boldsymbol{p})\odot\nabla f_i(\boldsymbol{x}^*)||^2 \\
        &\leq \frac{4}{NK}\sum_{i, k}\mathbb{E}||\boldsymbol{p}\odot(\nabla f_i(\boldsymbol{y}_{i,k}) - \nabla f_i(\boldsymbol{x}))||^2 + \frac{4}{N}\sum_{i}\mathbb{E}||\boldsymbol{p}\odot(\nabla f_i(\boldsymbol{x}) - \nabla f_i(\boldsymbol{x}^*))||^2 \\
        &\quad + \frac{2}{N}\sum_{i}\mathbb{E}||(\boldsymbol{1} - \boldsymbol{p})\odot\nabla f_i(\boldsymbol{x}^*)||^2 \\
        &\leq 4\beta_p^2\mathcal{E}_r + 8\beta_p(\mathbb{E}[f(\boldsymbol{x})] - f(\boldsymbol{x}^*)) + 2\zeta_{1-p}^2 \,. 
    \end{split}
\end{equation}
The first and second step uses Jensen inequality. The last step uses the smoothness and convexity of the function. The definition of the function heterogeneity $\zeta_{1-p}^2$ completes the proof.

\end{proof}

\vspace*{-6mm}

\begin{lemma}
Suppose $f_i$ satisfies the Assumption.~\ref{assum:mu_convex} to Assumption.~\ref{assum:sigma_bound}, then for any global $\eta_g\geq1$, we can bound the drift as:
\begin{equation}
    \mathcal{E}_r \leq \frac{36\tilde{\eta}^2}{\eta_g^2}C_{r-1} + \frac{18\tilde{\eta}^2\beta(\mathbb{E}[f(\boldsymbol{x}^{r-1})] - f(\boldsymbol{x}^*))}{\eta_g^2} + \frac{36\tilde{\eta}^2\sigma_p^2}{K\eta_g^2} + \frac{3\tilde{\eta}^2\sigma^2}{K\eta_g^2} \,.
\end{equation}
\label{lemma2_client_drift}
\end{lemma}

\begin{proof}
\vspace*{-10mm}
If $K=1$, the lemma is trivially true as the left-hand side $\mathcal{E}_r$ = 0 and the right hand size are positive. For $K > 1$, we build the bound for $\mathcal{E}_r$ recursively with dropping the round index:
\begin{equation}
    \begin{split}
        \mathcal{T}_1 &= \frac{1}{N}\sum_{i}\mathbb{E}[||\boldsymbol{y}_{i,k} - \boldsymbol{x}||^2] = \frac{1}{N}\sum_{i}\mathbb{E}[||\boldsymbol{y}_{i, k-1} - \eta_lg(\boldsymbol{y}_{i, k-1}) + \eta_l\boldsymbol{c}_i - \eta_l\boldsymbol{c} - \boldsymbol{x}||^2] \\
        &\leq \frac{1}{N}\sum_{i}\mathbb{E}[||\boldsymbol{y}_{i, k-1} - \eta_l\nabla f_i(\boldsymbol{y}_{i, k-1}) + \eta_l\nabla f_i(\boldsymbol{x}) - \eta_l\nabla f_i(\boldsymbol{x}) - \boldsymbol{x} +\eta_l \boldsymbol{c}_i - \eta_l\boldsymbol{c}||^2] + \eta_l^2\sigma^2 \\
        &\overset{\ref{eq:triangle_2}}{\leq} \underbrace{\frac{1+a}{N}\sum_{i}\mathbb{E}[||\boldsymbol{y}_{i, k-1} - \eta_l\nabla f_i(\boldsymbol{y}_{i, k-1}) + \eta_l \nabla f_i(\boldsymbol{x}) - \boldsymbol{x}||^2]}_{\mathcal{T}_2} \\
        &\quad + (1+a^{-1})\eta_l^2\underbrace{\frac{1}{N}\sum_{i}\mathbb{E}[||\boldsymbol{c}_i - \nabla f_i(\boldsymbol{x}) - \boldsymbol{c}||^2]}_{\mathcal{T}_3} + \eta_l^2\sigma^2\,.
    \end{split}
\end{equation}

\begin{equation}
    \mathcal{T}_2 \overset{~\ref{eq:contrastive_mapping}}{\leq} \frac{1+a}{N}\sum_{i}\mathbb{E}[||\boldsymbol{y}_{i, k-1} - \boldsymbol{x}||^2]\,.
\end{equation}

\begin{equation}
    \begin{split}
        \mathcal{T}_3 &\leq \frac{3}{N}\sum_{i}\mathbb{E}||\boldsymbol{c}_i - \nabla f_i(\boldsymbol{x}^*)||^2 + 3\mathbb{E}||\boldsymbol{c}||^2 + \frac{3}{N}\sum_{i}\mathbb{E}||\nabla f_i(\boldsymbol{x}) - \nabla f_i(\boldsymbol{x}^*)||^2 \\ 
        &\leq \frac{6}{N}\sum_{i}\mathbb{E}||\boldsymbol{c}_i - \nabla f_i(\boldsymbol{x}^*)||^2 + \frac{3}{N}\sum_{i}\mathbb{E}||\nabla f_i(\boldsymbol{x}) - \nabla f_i(\boldsymbol{x}^*)||^2 \\
        &\leq \frac{12}{N}\sum_{i}\mathbb{E}||\mathbb{E}[\boldsymbol{c}_i] - \nabla f_i(\boldsymbol{x}^*)||^2 + \frac{3}{N}\sum_{i}\mathbb{E}||\nabla f_i(\boldsymbol{x}) - \nabla f_i(\boldsymbol{x}^*)||^2 + \frac{12\sigma_p^2}{K} \\
        &\leq 12C_{r-1} + 6\beta(\mathbb{E}[f(\boldsymbol{x})] - f(\boldsymbol{x}^*)) + \frac{12\sigma_p^2}{K} \,.
    \end{split}
\end{equation}

The first step uses Jensen inequality. The second steps uses the definition of $\boldsymbol{c} = \frac{1}{N}\sum_{i}\boldsymbol{c}_i$. The last step uses the smoothness and convexity of the function and the definition of $C_{r-1}$. Then using $a = \frac{1}{K-1}$, we can simplify $\mathcal{T}_1$ as:
\begin{equation}
\begin{split}
    \mathcal{T}_1 &= \frac{1}{N}\sum_{i}\mathbb{E}||\boldsymbol{y}_{i,k} - \boldsymbol{x}||^2 \\
    &\leq (1+\frac{1}{K-1})\frac{1}{N}\sum_{i}\mathbb{E}||\boldsymbol{y}_{i,k-1} - \boldsymbol{x}||^2 + 12K\eta_l^2C_{r-1} + 6K\eta_l^2\beta(\mathbb{E}[f(\boldsymbol{x})] - f(\boldsymbol{x}^*)) \\
    &\quad + 12\eta_l^2\sigma_p^2 + \eta_l^2\sigma^2 \,.
\end{split}
\label{eq:lemma_2_unrol}
\end{equation}

Unrolling the recursion~\ref{eq:lemma_2_unrol}, we get the following for any $k\in \{1, 2, ..., K\}$
\begin{equation}
\begin{split}
    \mathcal{T}_1 &\leq (12K\eta_l^2C_{r-1} + 6K\eta_l^2\beta(\mathbb{E}[f(\boldsymbol{x})] - f(\boldsymbol{x}^*)) + 12\eta_l^2\sigma_p^2 + \eta_l^2\sigma^2)\sum_{\tau=0}^k(1+\frac{1}{K-1})^\tau \\
    &\leq (12K\eta_l^2C_{r-1} + 6K\eta_l^2\beta(\mathbb{E}[f(\boldsymbol{x})] - f(\boldsymbol{x}^*)) + 12\eta_l^2\sigma_p^2 + \eta_l^2\sigma^2)3K \\
    &\leq \frac{36\tilde{\eta}^2}{\eta_g^2}C_{r-1} + \frac{18\tilde{\eta}^2\beta(\mathbb{E}[f(\boldsymbol{x})] - f(\boldsymbol{x}^*))}{\eta_g^2} + \frac{36\tilde{\eta}^2\sigma_p^2}{K\eta_g^2} + \frac{3\tilde{\eta}^2\sigma^2}{K\eta_g^2} \,.
\end{split}
\end{equation}

Averaging over $K$ yields the lemma statement.
\vspace*{-1mm}
\end{proof}
\newpage
\begin{lemma}
\textbf{Progress in one round}: Assume assumptions~\ref{assum:mu_convex}-\ref{assum:beta_smooth} are true. The following hold for any step size satisfying $\tilde{\eta} \leq \text{min}\{\frac{1}{80\beta}, \frac{26}{20\mu} \}$, $\eta_g\geq 1$ and $\tilde{\eta} := K\eta_l\eta_g$ :

\begin{equation}
\begin{split}
    \mathbb{E}||\boldsymbol{x}^{r} - \boldsymbol{x}^*||^2 + 18\tilde{\eta}^2C_r &\leq (1-\frac{\mu\eta}{2})(\mathbb{E}||\boldsymbol{x}^{r-1} - \boldsymbol{x}^*||^2 + 18\tilde{\eta}^2C_{r-1}) - \tilde{\eta}(\mathbb{E}[f(\boldsymbol{x}^{r-1})] - f(\boldsymbol{x}^*)) \\
    &\quad + 36\tilde{\eta}^2\zeta_{1-p}^2 + \frac{\tilde{\eta}^2\sigma^2}{KN}(24 + \frac{117N}{80\eta_g^2}) \,.
\end{split}
\end{equation}
\label{lemma3_in_one_round}
\end{lemma}

\vspace*{-3mm}

\begin{proof}

Starting from the server update equation
\begin{equation}
    \Delta x= -\frac{\tilde{\eta}}{KN}\sum_{i,k}(g_i(\boldsymbol{y}_{i, k-1}) - \boldsymbol{c}_i + \boldsymbol{c}), \quad \mathbb{E}[\Delta x] = -\frac{\tilde{\eta}}{KN}\sum_{i,k}g_i(\boldsymbol{y}_{i, k-1})\,.
\end{equation}

Using Lemma~\ref{lemma1_variance} that bounds $\mathbb{E}||\Delta x||^2$, we can write:
\begin{equation}
    \begin{split}
        \mathbb{E}||\boldsymbol{x} + \Delta \boldsymbol{x} - \boldsymbol{x}^*||^2 &= \mathbb{E}||\boldsymbol{x}-\boldsymbol{x}^*||^2 + \mathbb{E}||\Delta \boldsymbol{x}||^2 - \frac{2\tilde{\eta}}{KN}\mathbb{E}\sum_{i, k}<\boldsymbol{x}-\boldsymbol{x}^*, \nabla f_{i}(\boldsymbol{y}_{i, k-1})> \\
        &\leq \mathbb{E}||\boldsymbol{x}-\boldsymbol{x}^*||^2 \underbrace{- \frac{2\tilde{\eta}}{KN}\sum_{i, k}\mathbb{E}<\boldsymbol{x}-\boldsymbol{x}^*, \nabla f_{i}(\boldsymbol{y}_{i, k-1})>}_{\mathcal{T}_5} + 8\tilde{\eta}^2\beta^2\mathcal{E}_r \\
        &\quad + 8\tilde{\eta}^2\beta (\mathbb{E}[f(\boldsymbol{x})] - f(\boldsymbol{x}^*)) + 16\tilde{\eta}^2C_{r-1} + \frac{8\tilde{\eta}^2\sigma^2}{KN} + \frac{16\tilde{\eta}^2\sigma_p^2}{KN} \,.\\
    \end{split}
\end{equation}

Given Lemma.~\ref{lemma:perturbed_strong_convexity}, let $h=f_i, \boldsymbol{z}=\boldsymbol{x}^*, \boldsymbol{y}=\boldsymbol{x}, \boldsymbol{x}=\boldsymbol{y}_{i, k-1}^r$, then
\begin{equation}
    \begin{split}
    \mathcal{T}_5 &= \frac{2\tilde{\eta}}{KN}\mathbb{E}\sum_{i, k}\langle \nabla f_i(\boldsymbol{y}_{i, k}), \boldsymbol{x}^*-\boldsymbol{x}\rangle \\
    &\leq \frac{2\tilde{\eta}}{KN}\mathbb{E}\sum_{i, k} f_i(\boldsymbol{x}^*) - f_i(\boldsymbol{x}) - \frac{\mu}{4}||\boldsymbol{x} - \boldsymbol{x}^*||^2 + \beta||\boldsymbol{y}_{i, k-1}^r - \boldsymbol{x}||^2\\
        &= -2\tilde{\eta}(\mathbb{E}[f(\boldsymbol{x})] - f(\boldsymbol{x}^*)) - \frac{\tilde{\eta}\mu}{2}\mathbb{E}||\boldsymbol{x} - \boldsymbol{x}^*||^2 + \frac{2\tilde{\eta}\beta}{KN}\sum_{i, k}\mathbb{E}||\boldsymbol{y}_{i, k-1}^r - \boldsymbol{x}||^2 \\
        &= -2\tilde{\eta}(\mathbb{E}[f(\boldsymbol{x})] - f(\boldsymbol{x}^*)) - \frac{\tilde{\eta}\mu}{2}\mathbb{E}||\boldsymbol{x} - \boldsymbol{x}^*||^2 + 2\tilde{\eta}\beta\mathcal{E}_r\,.
    \end{split}
\end{equation}
Plugging $\mathcal{T}_5$ back to the equation, we obtain:
\begin{equation}
    \begin{split}
        \mathbb{E}||\boldsymbol{x}^r - \boldsymbol{x}^*||^2 
        &\leq (1-\frac{\tilde{\eta}\mu}{2})(\mathbb{E}[||\boldsymbol{x} - \boldsymbol{x}^*||^2]) + (8\tilde{\eta}^2\beta - 2\tilde{\eta})(\mathbb{E}[f(\boldsymbol{x})] - f(\boldsymbol{x}^*)) \\
        &\quad + (2\tilde{\eta}\beta + 8\tilde{\eta}^2\beta^2)\mathcal{E}_r + 16\tilde{\eta}^2C_{r-1} + \frac{8\tilde{\eta}^2\sigma^2}{KN} + \frac{16\tilde{\eta}^2\sigma_p^2}{KN}\,. \\
    \end{split}
\end{equation}

According to Lemma~\ref{lemma:cr}, 
\begin{equation}
    18\tilde{\eta}^2C_{r} \leq 72\beta_p^2\tilde{\eta}^2\mathcal{E}_r + 144\beta_p\tilde{\eta}^2(\mathbb{E}[f(\boldsymbol{x})] - f(\boldsymbol{x}^*)) + 36\tilde{\eta}^2\zeta_{1-p}^2\,.
\end{equation}

According to Lemma~\ref{lemma2_client_drift}, 
\begin{equation}
    \begin{split}
        3\tilde{\eta}\beta\mathcal{E}_r &\leq 3\tilde{\eta}\beta(\frac{36\tilde{\eta}^2}{\eta_g^2}C_{r-1} + \frac{18\tilde{\eta}^2\beta(\mathbb{E}[f(\boldsymbol{x})] - f(\boldsymbol{x}^*))}{\eta_g^2} + \frac{36\tilde{\eta}^2\sigma_p^2}{K\eta_g^2} + \frac{3\tilde{\eta}^2\sigma^2}{K\eta_g^2}) \\
        &\leq \frac{108\tilde{\eta}^2\tilde{\eta}\beta}{\eta_g^2}C_{r-1} + \frac{54\tilde{\eta}\tilde{\eta}^2\beta^2(\mathbb{E}[f(\boldsymbol{x})] - f(\boldsymbol{x}^*))}{\eta_g^2} + \frac{108\tilde{\eta}^2\tilde{\eta}\beta \sigma_p^2}{K\eta_g^2} + \frac{9\tilde{\eta}^2\tilde{\eta}\beta\sigma^2}{K\eta_g^2} \\
        &\leq \frac{27\tilde{\eta}^2C_{r-1}}{20} + \frac{7\tilde{\eta}(\mathbb{E}[f(\boldsymbol{x})] - f(\boldsymbol{x}^*))}{800\eta_g^2} + \frac{27\tilde{\eta}^2\sigma_p^2}{20K\eta_g^2} + \frac{9\tilde{\eta}^2\sigma^2}{80K\eta_g^2}\,.
    \end{split}
\end{equation}

The last step follows $\tilde{\eta} \leq \frac{1}{80\beta}$.

Add the above three equations together:
\begin{equation}
    \begin{split}
        \mathbb{E}||\boldsymbol{x}^{r} - \boldsymbol{x}^*||^2 + 18\tilde{\eta}^2C_r &\leq (1-\frac{\tilde{\eta}\mu}{2})(\mathbb{E}||\boldsymbol{x}^{r-1} - \boldsymbol{x}^*||^2 + 18\tilde{\eta}^2C_{r-1}) \\
        &\quad + (8\tilde{\eta}^2\beta - 2\tilde{\eta} + 144\beta_p\tilde{\eta}^2 + \frac{7\tilde{\eta}}{800\eta_g^2})(\mathbb{E}[f(\boldsymbol{x}^{r-1})] - f(\boldsymbol{x}^*))\\
        &\quad + (8\tilde{\eta}^2\beta^2 + 72\beta_p^2\tilde{\eta}^2-\tilde{\eta}\beta)\mathcal{E}_r + 36\tilde{\eta}^2\zeta_{1-p}^2 \\
        &\quad + (\frac{\tilde{\eta}\mu}{2} - \frac{13}{20})\tilde{\eta}^2C_{r-1} \\
        &\quad + \frac{\tilde{\eta}^2\sigma^2}{KN}(8 + \frac{9N}{80\eta_g^2}) + \frac{\tilde{\eta}^2\sigma_p^2}{KN}(16 + \frac{27N}{20\eta_g^2}) \\
        &\leq (1-\frac{\tilde{\eta}\mu}{2})(\mathbb{E}||\boldsymbol{x}^{r-1} - \boldsymbol{x}^*||^2 + 18\tilde{\eta}^2C_{r-1}) - \tilde{\eta}(\mathbb{E}[f(\boldsymbol{x}^{r-1})] - f(\boldsymbol{x}^*)) \\
        &\quad + 36\tilde{\eta}^2\zeta_{1-p}^2 + \frac{\tilde{\eta}^2\sigma^2}{KN}(24 + \frac{117N}{80\eta_g^2}) \,.
    \end{split}
\end{equation}

As $0 \leq \beta_p \leq \beta$, the lemma follows from noting that $\tilde{\eta} \leq \frac{1}{80\beta}$ implies $(8\tilde{\eta}^2\beta^2 + 72\beta_p^2\tilde{\eta}^2 - \tilde{\eta}\beta) \leq 0$, $(8\tilde{\eta}^2\beta - 2\tilde{\eta} + 144\beta_p\tilde{\eta}^2) \leq 0$ and $\tilde{\eta} \leq \frac{26}{20\mu}$ implies $(\frac{\tilde{\eta}\mu}{2} - \frac{13}{20}) \leq 0$. The last step uses the preposition that $\sigma_p^2 \leq \sigma^2$.

\end{proof}
\textbf{The final rate for the general convex case} 
For the general convex case ($\mu=0$), we proceed by unrolling the recursive bound in Lemma.~\ref{lemma3_in_one_round}:
\begin{equation}
    \begin{split}
        \frac{1}{R}\sum_{r=1}^R\mathbb{E}[f(\boldsymbol{x}^{r-1})] - f(\boldsymbol{x}^*) &\leq \frac{||\boldsymbol{x}^0 - \boldsymbol{x}^*||^2}{\tilde{\eta}R} + \frac{18\tilde{\eta}C_0}{R} + 36\tilde{\eta}\zeta_{1-p}^2 + \frac{\tilde{\eta}\sigma^2}{KN}(24 + \frac{117N}{80\eta_g^2}) \,. \\
    \end{split}
\end{equation}

Using Lemma~\ref{lemma:tune_stepsize}, let $D:=||\boldsymbol{x}^0-\boldsymbol{x}^*||^2$, $F:=f(\boldsymbol{x}^0) - f(\boldsymbol{x}^*)$, and $\tilde{\eta} \leq \frac{1}{80\beta}$ :
\begin{equation}
\mathbb{E}[f(\bar{\boldsymbol{x}}^R)] - f(\boldsymbol{x}^*) \leq \mathcal{O}\left(\frac{\sigma \sqrt{D}}{\sqrt{RKN}}\sqrt{1+\frac{N}{\eta_g^2}} + \frac{\zeta_{1-p}\sqrt{D}}{\sqrt{R}} + \frac{\beta D}{R} + F\right)  \,.  
\end{equation}

\textbf{The final rate for the strongly convex case}. For the strongly convex case ($\mu>0$), we use Lemma~\ref{lemma:stich_stepsize} letting $r^{t} = ||\boldsymbol{x}^t - \boldsymbol{x}^*||^2$, $a = \frac{\mu}{2}$, $b=1$, $c=\zeta_{1-p}^2 + \frac{\sigma^2}{KN}(1+\frac{N}{\eta_g^2})$, and $\frac{1}{d}=\text{min}\left(\frac{1}{80\beta}, \frac{26}{20\mu}\right)$, $R\geq \text{max}\{\frac{20}{13}, \frac{160\beta}{\mu}\}$, $D=||\boldsymbol{x}^0-\boldsymbol{x}^*||^2$ we obtain:

\begin{equation}
    \mathbb{E}[f(\bar{\boldsymbol{x}}^R)] - f(\boldsymbol{x}^*) \leq \tilde{\mathcal{O}}\left(\frac{\sigma^2}{\mu NKR}(1+\frac{N}{\eta_g^2}) + \frac{\zeta_{1-p}^2}{\mu R} + \mu D\text{exp}\left(-\text{min}\left\{\frac{13}{20}, \frac{\mu}{160\beta}\right\}R \right) \right) \,.
\end{equation}

\begin{remark}
Following~\cite{DBLP:journals/corr/abs-1910-06378}, we can bound $C_0$ with using warm-up strategy by $f(\boldsymbol{x^0}) - f(\boldsymbol{x}^*)$ and $\sigma^2$, so to simplify the convergence rate, we simply represent $C_0$ with $F$ in the general convex case. For the strongly convex case, as the exponential term has a bigger influence on the convergence rate, we simply omit $C_0$ in the coefficient of the exponential term.  
\end{remark}

\subsection{Proof for the convergence rate for non-convex functions}
We follow the same procedure to derive the convergence rate for the non-convex functions. We first bound the variance of the server update in Lemma.~\ref{lemma:var_non_convex} and then the client drift in Lemma.~\ref{lemma:client_drift_non_convex}. We then combine the results from these two lemmas to get the bound for the progress in one round in Lemma.~\ref{lemma:progress_in_one_round_non_convex}. Following that, we give the convergence rate. To proceed the proof, we make the following definitions:

Recall the definition for client drift $\mathcal{E}_r$ and $\boldsymbol{c}_i$:
\begin{subequations}
\begin{equation*}
    \mathcal{E}_r := \frac{1}{NK}\sum_{i=1}^N\sum_{k=1}^K\mathbb{E}||\boldsymbol{y}_{i,k}^r - \boldsymbol{x}^{r-1}||^2 
\end{equation*}
\begin{equation*}
    \boldsymbol{c}_i^r = \frac{1}{K}\sum_{k=1}^K\boldsymbol{p}\odot g_i(\boldsymbol{y}_{i,k}^r)
\end{equation*}
\end{subequations}

\begin{definition}
We define $\Theta_r$ as:
\begin{equation}
    \Theta_r := \frac{1}{N}\sum_{i=1}^N \mathbb{E}||\mathbb{E}[\boldsymbol{c}_i] - \boldsymbol{p}\odot\nabla f_i(\boldsymbol{x})||^2 \,.
\end{equation}
\end{definition}

\begin{lemma}
\label{lemma:var_non_convex}
For updates from~\ref{eq:update_0} to~\ref{eq:update_2}, the following holds true for any $\tilde{\eta} := \eta_l\eta_gK$
\begin{equation}
    \mathbb{E}||\boldsymbol{x}^r - \boldsymbol{x}^{r-1}||^2 \leq 8\tilde{\eta}^2\beta^2\mathcal{E}_r + 16\tilde{\eta}^2\Theta_{r-1} + 8\tilde{\eta}^2\mathbb{E}||\nabla f(\boldsymbol{x})||^2 + \frac{6\tilde{\eta}^2\sigma^2}{NK} + \frac{12\tilde{\eta}^2\sigma_p^2}{NK} \,.
\end{equation}
\end{lemma}

\begin{proof}
\begin{equation}
    \begin{split}
        \mathbb{E}||\Delta \boldsymbol{x}||^2 &= \mathbb{E}||\frac{\tilde{\eta}}{NK}\sum_{i,k}g_i(\boldsymbol{y}_{i,k}^{r}) - \boldsymbol{c}_i + \boldsymbol{c}||^2 \\
        &\leq 2\mathbb{E}||\frac{\tilde{\eta}}{NK}\sum_{i,k}\nabla f_i(\boldsymbol{y}_{i,k}) - \mathbb{E}[\boldsymbol{c}_i] + \mathbb{E}[\boldsymbol{c}]||^2 + 2(\frac{3\tilde{\eta}^2\sigma^2}{KN} + \frac{6\tilde{\eta}^2\sigma_p^2}{KN}) \\
        &\leq 2\mathbb{E}||\frac{\tilde{\eta}}{NK}\sum_{i,k}\nabla f_i(\boldsymbol{y}_{i,k}) - \nabla f_i(\boldsymbol{x}) + \boldsymbol{p}\odot\nabla f_i(\boldsymbol{x}) - \mathbb{E}[\boldsymbol{c}_i] \\
        &\quad +\mathbb{E}[\boldsymbol{c}] - \boldsymbol{p}\odot\nabla f(\boldsymbol{x}) + (\boldsymbol{1} -\boldsymbol{p})\odot\nabla f_i(\boldsymbol{x}) + \boldsymbol{p}\odot \nabla f(\boldsymbol{x})||^2 + \frac{6\tilde{\eta}^2\sigma^2}{NK}+\frac{12\tilde{\eta}^2\sigma_p^2}{NK} \\
        &\leq \frac{8\tilde{\eta}^2}{NK}\sum_{i,k}\mathbb{E}||\nabla f_i(\boldsymbol{y}_{i,k}) - \nabla f_i(\boldsymbol{x})||^2 \\
        &\quad + \frac{16\tilde{\eta}^2}{N}\sum_{i}\mathbb{E}||\mathbb{E}[\boldsymbol{c}_i] - \boldsymbol{p}\odot\nabla f_i(\boldsymbol{x})||^2 + 8\tilde{\eta}^2\mathbb{E}||\nabla f(\boldsymbol{x})||^2 + \frac{6\tilde{\eta}^2\sigma^2}{NK}+\frac{12\tilde{\eta}^2\sigma_p^2}{NK} \\
        &\leq 8\tilde{\eta}^2\beta^2\mathcal{E}_r + 16\tilde{\eta}^2\Theta_{r-1} + 8\tilde{\eta}^2\mathbb{E}||\nabla f(\boldsymbol{x})||^2 + \frac{6\tilde{\eta}^2\sigma^2}{NK} + \frac{12\tilde{\eta}^2\sigma_p^2}{NK} \,.
    \end{split}
\end{equation}
The second step uses Lemma~\ref{lemma:separate_mean_var}. The last second step uses Jensen inequality. Given the definition of $\mathcal{E}_r$ and $\Theta_{r}$, we complete the proof.

\end{proof}

\begin{lemma}
\label{lemma:client_drift_non_convex}
Suppose $f_i$ satisfy Assumption.~\ref{assum:beta_smooth},~\ref{assum:zeta_hat}, and Assumption.~\ref{assum:sigma_bound}, then, for any global $\eta_g \geq 1$, we can bound the drift as:
\begin{subequations}
\begin{equation}
    \mathcal{E}_r \leq \frac{60\tilde{\eta}^2\Theta_{r-1}}{\eta_g^2} + \frac{15\tilde{\eta}^2\mathbb{E}||\nabla f(\boldsymbol{x})||^2}{\eta_g^2} + \frac{15\tilde{\eta}^2\hat{\zeta}_{1-p}^2}{\eta_g^2} +  \frac{60\tilde{\eta}^2\sigma_p^2}{K\eta_g^2} + \frac{3\tilde{\eta}^2\sigma^2}{K\eta_g^2},
\end{equation}
\begin{equation}
    \Theta_r \leq \beta_p^2\mathcal{E}_r\,.
\end{equation}
\end{subequations}
\end{lemma}

\begin{proof}
We first observe that if $K=1$, then $\mathcal{E}_r = 0$ since $\boldsymbol{y}_{i,0} = \boldsymbol{x}$ for all $i\in [N]$ and the right hand side are all positive. Thus, the lemma is trivially true if $K=1$ and we will assume $K>1$ for the following proof. Starting from the update rule for $i \in [N], k\in [K]$

\begin{equation}
    \begin{split}
        \mathbb{E}||\boldsymbol{y}_{i,k} - \boldsymbol{x}||^2 &= \mathbb{E}||\boldsymbol{y}_{i,k-1} - \eta_l(\nabla f_i(\boldsymbol{y}_{i,k-1}) - \boldsymbol{c}_i + \boldsymbol{c}) - \boldsymbol{x}||^2 + \eta_l^2\sigma^2 \\
        &\leq (1+\frac{1}{K-1})\mathbb{E}||\boldsymbol{y}_{i,k-1} - \boldsymbol{x}||^2 + K\eta_l^2\underbrace{\mathbb{E}||\nabla f_i(\boldsymbol{y}_{i,k}) - \boldsymbol{c}_i + \boldsymbol{c}||^2}_{\mathcal{T}_6} + \eta_l^2\sigma^2\,.
    \end{split}
\end{equation}
The first step uses the Lemma.~\ref{lemma:triangle} with $\alpha=K-1$

\begin{equation}
    \begin{split}
        \mathcal{T}_6 &= \mathbb{E}||\boldsymbol{p}\odot\nabla f_i(\boldsymbol{y}_{i,k-1}) - \boldsymbol{p}\odot\nabla f_i(\boldsymbol{x}) - \boldsymbol{c}_i + \boldsymbol{p}\odot \nabla f_i(\boldsymbol{x}) + \boldsymbol{c} - \boldsymbol{p}\odot \nabla f(\boldsymbol{x}) + \boldsymbol{p}\odot \nabla f(\boldsymbol{x}) \\
        &\quad + (\boldsymbol{1} - \boldsymbol{p})\odot \nabla f_i(\boldsymbol{y}_{i,k-1})||^2 \\
        &\leq 5\mathbb{E}||\boldsymbol{p}\odot(\nabla f_i(\boldsymbol{y}_{i,k-1}) - \nabla f_i(\boldsymbol{x}))||^2 + 5\mathbb{E}||c_i - \boldsymbol{p}\odot\nabla f_i(\boldsymbol{x})||^2 + \\
        &\quad+ 5\mathbb{E}||\boldsymbol{c} -  \boldsymbol{p}\odot \nabla f(\boldsymbol{x})||^2 + 5\mathbb{E}||\boldsymbol{p}\odot \nabla f(\boldsymbol{x})||^2 + 5\mathbb{E}||(\boldsymbol{1} - \boldsymbol{p})\odot\nabla f_i(\boldsymbol{y}_{i,k-1})||^2\\
        &\leq 5\mathbb{E}||\boldsymbol{p}\odot(\nabla f_i(\boldsymbol{y}_{i,k-1}) - \nabla f_i(\boldsymbol{x}))||^2 + 20\mathbb{E}||\mathbb{E}[\boldsymbol{c}_i] - \boldsymbol{p}\odot \nabla f_i(\boldsymbol{x})||^2 \\
        &\quad+ 5\mathbb{E}||\boldsymbol{p}\odot\nabla f(\boldsymbol{x})||^2 + 5\mathbb{E}||(\boldsymbol{1} - \boldsymbol{p})\odot\nabla f_i(\boldsymbol{y}_{i,k-1})||^2+ \frac{20\sigma_p^2}{K}\\ 
        &\leq 5\beta_p^2\mathbb{E}||\boldsymbol{y}_{i,k-1} - \boldsymbol{x}||^2 + 20\mathbb{E}||\mathbb{E}[\boldsymbol{c}_i] - \boldsymbol{p}\odot\nabla f_i(\boldsymbol{x})||^2 + 5\mathbb{E}||\nabla f(\boldsymbol{x})||^2 \\
        &\quad + 5\mathbb{E}||(\boldsymbol{1} - \boldsymbol{p})\odot\nabla f_i(\boldsymbol{y}_{i,k-1})||^2 + \frac{20\sigma_p^2}{K} \,.
    \end{split}
\end{equation}
The second step uses Jensen inequality. The third step uses the Lemma~\ref{lemma:separate_mean_var}. The last step uses the smoothness of the function. 

\begin{equation}
    \begin{split}
        \frac{1}{N}\sum_{i}\mathbb{E}||\boldsymbol{y}_{i,k} - \boldsymbol{x}||^2 &\leq (1+\frac{1}{K-1}+5K\eta_l^2\beta_p^2)\frac{1}{N}\sum_{i}\mathbb{E}||\boldsymbol{y}_{i,k-1} - \boldsymbol{x}||^2 \\
        &\quad + 20K\eta_l^2\underbrace{\frac{1}{N}\sum_{i}\mathbb{E}||\mathbb{E}[\boldsymbol{c}_i] -\boldsymbol{p}\odot \nabla f_i(\boldsymbol{x})||^2}_{\Theta_{r-1}} \\
        &\quad + 5K\eta_l^2\mathbb{E}||\nabla f(\boldsymbol{x})||^2 + 5K\eta_l^2\hat{\zeta}_{1-p}^2\\
        &\quad + 20\eta_l^2\sigma_p^2 + \eta_l^2\sigma^2 \,.
    \end{split}
\end{equation}
The above step uses the definition of $\hat{\zeta}_{1-p}^2$. We next bound $\Theta_{r}$ using the smoothness Assumption.~\ref{assum:beta_smooth}.
\begin{equation}
    \begin{split}
        \Theta_{r-1} &= \frac{1}{N}\sum_{i}\mathbb{E}||\frac{1}{K}\sum_{k}\boldsymbol{p}\odot\nabla f_i(\boldsymbol{y}_{i,k}) - \boldsymbol{p}\odot\nabla f_i(\boldsymbol{x})||^2 \\
        &\leq \frac{1}{NK}\sum_{i,k}\mathbb{E}||\boldsymbol{p}\odot(\nabla f_i(\boldsymbol{y}_{i,k}) - \nabla f_i(\boldsymbol{x}))||^2\\
        &\leq \beta_p^2\mathcal{E}_{r-1} \,.
    \end{split}
\end{equation}
Use the definition of $\Theta_r$, we can then bound $\frac{1}{N}\sum_{i}\mathbb{E}||\boldsymbol{y}_{i,k} - \boldsymbol{x}||^2$:
\begin{equation}
    \begin{split}
        \frac{1}{N}\sum_{i}\mathbb{E}||\boldsymbol{y}_{i,k} - \boldsymbol{x}||^2 &\leq \left(1+\frac{1}{K-1}+\frac{5\tilde{\eta}^2\beta_p^2}{K\eta_g^2}\right)\frac{1}{N}\sum_{i}\mathbb{E}||\boldsymbol{y}_{i,k-1} - \boldsymbol{x}||^2 + 20K\eta_l^2\Theta_{r-1} \\
        &\quad + 5K\eta_l^2\mathbb{E}||\nabla f(\boldsymbol{x})||^2 + 5K\eta_l^2\hat{\zeta}_{1-p}^2 + 20\eta_l^2\sigma_p^2 + \eta_l^2\sigma^2 \\
        &\leq (20K\eta_l^2\Theta_{r-1} + 5K\eta_l^2\mathbb{E}||\nabla f(\boldsymbol{x})||^2+5K\eta_l^2\hat{\zeta}_{1-p}^2 \\
        &\quad + 20\eta_l^2\sigma_p^2+\eta_l^2\sigma^2)\sum_{\tau=0}^{k-1}\left(1+\frac{1}{K-1}+\frac{5\tilde{\eta}^2\beta_p^2}{K\eta_g^2}\right)^\tau \\
        &\leq (20K\eta_l^2\Theta_{r-1} + 5K\eta_l^2\mathbb{E}||\nabla f(\boldsymbol{x})||^2+5K\eta_l^2\hat{\zeta}_{1-p}^2 + 20\eta_l^2\sigma_p^2+\eta_l^2\sigma^2)3K \\
        &\leq \frac{60\tilde{\eta}^2\Theta_{r-1}}{\eta_g^2} + \frac{15\tilde{\eta}^2\mathbb{E}||\nabla f(\boldsymbol{x})||^2}{\eta_g^2} + \frac{15\tilde{\eta}^2\hat{\zeta}_{1-p}^2}{\eta_g^2} +  \frac{60\tilde{\eta}^2\sigma_p^2}{K\eta_g^2} + \frac{3\tilde{\eta}^2\sigma^2}{K\eta_g^2} \,.
    \end{split}
\end{equation}
Averaging over $K$ yields the lemma statement. 
\end{proof}

\begin{lemma}
\label{lemma:progress_in_one_round_non_convex}
Suppose the updates~\ref{eq:update_0} to~\ref{eq:update_2} satisfy the smooth assumption and bounded variance. For any effective step size $\tilde{\eta}:=K\eta_l\eta_g$ satisfying $\tilde{\eta} \leq \frac{1}{26\beta}$:
\begin{equation}
    \begin{split}
        \mathbb{E}[f(\boldsymbol{x} + \Delta \boldsymbol{x})] + 9\beta\tilde{\eta}^2\Theta_r &\leq \mathbb{E}[f(\boldsymbol{x})] + 9\beta\tilde{\eta}^2\Theta_{r-1} -\tilde{\eta}\mathbb{E}||\nabla f(\boldsymbol{x})||^2 \\
        &\quad +\frac{\tilde{\eta}^2\beta\sigma^2}{KN}(\frac{63N}{26\eta_g^2} + 9) + \frac{15\tilde{\eta}^2\beta\hat{\zeta}_{1-p}^2}{26\eta_g^2} \,.
    \end{split}
\end{equation}

\end{lemma}

\begin{proof}
\begin{equation}
    \mathbb{E}_{r-1}[f(\boldsymbol{x}+\Delta \boldsymbol{x})] \leq f(\boldsymbol{x})+\langle \nabla f(\boldsymbol{x}), \mathbb{E}_{r-1}[\Delta \boldsymbol{x}]\rangle + \frac{\beta}{2}\mathbb{E}_{r-1}||\Delta \boldsymbol{x}||^2 \,.
\end{equation}
Recall that:
\begin{equation}
    \mathbb{E}[\Delta \boldsymbol{x}] = -\frac{\tilde{\eta}}{KN}\sum_{i,k}\nabla f_i(\boldsymbol{y}_{i,k})\,.
\end{equation}
Therefore
\begin{equation}
    \begin{split}
        \mathbb{E}[f(\boldsymbol{x}+\Delta \boldsymbol{x})] - f(\boldsymbol{x}) &\leq -\frac{\tilde{\eta}}{KN}\sum_{i,k}\langle \nabla f(\boldsymbol{x}), \nabla f_i(\boldsymbol{y}_{i, k})\rangle \\
        &\quad + \frac{\beta}{2}(8\tilde{\eta}^2\beta^2\mathcal{E}_r + 16\tilde{\eta}^2\Theta_{r-1} + 8\tilde{\eta}^2\mathbb{E}||\nabla f(\boldsymbol{x})||^2 + \frac{6\tilde{\eta}^2\sigma^2}{KN} + \frac{12\tilde{\eta}^2\sigma_p^2}{KN}) \\
        &\leq - \frac{\tilde{\eta}}{2}||\nabla f(\boldsymbol{x})||^2 + \frac{\tilde{\eta}}{2}\mathbb{E}||\frac{1}{NK}\sum_{i,k}\nabla f_i(\boldsymbol{y}_{i,k}) - \nabla f(\boldsymbol{x})||^2 \\
        &\quad + 4\tilde{\eta}^2\beta^3\mathcal{E}_r + 8\tilde{\eta}^2\beta \Theta_{r-1} + 4\tilde{\eta}^2\beta\mathbb{E}||\nabla f(\boldsymbol{x})||^2 + \frac{3\tilde{\eta}^2\beta\sigma^2}{KN} + \frac{6\tilde{\eta}^2\beta\sigma_p^2}{KN} \\
        &\leq (4\tilde{\eta}^2\beta - \frac{\tilde{\eta}}{2})\mathbb{E}||\nabla f(\boldsymbol{x})||^2 + (\frac{\tilde{\eta}\beta^2}{2} + 4\tilde{\eta}^2\beta^3)\mathcal{E}_r + 8\tilde{\eta}^2\beta \Theta_{r-1} + \frac{3\tilde{\eta}^2\beta\sigma^2}{KN} + \frac{6\tilde{\eta}^2\beta\sigma_p^2}{KN}\,.
    \end{split}
\end{equation}
The second step uses the inequality that $-ab = \frac{1}{2}((b-a)^2 - a^2) - \frac{1}{2}b^2 \leq \frac{1}{2}((b-a)^2 - a^2)$ for any $a, b\in \mathbb{R}$

\begin{equation}
    \begin{split}
        \tilde{\eta}\beta^2\mathcal{E}_r &\leq \tilde{\eta}\beta^2\left(\frac{60\tilde{\eta}^2\Theta_{r-1}}{\eta_g^2} + \frac{15\tilde{\eta}^2\mathbb{E}||\nabla f(\boldsymbol{x})||^2}{\eta_g^2} + \frac{15\tilde{\eta}^2\hat{\zeta}_{1-p}^2}{\eta_g^2} +  \frac{60\tilde{\eta}^2\sigma_p^2}{K\eta_g^2} + \frac{3\tilde{\eta}^2\sigma^2}{K\eta_g^2} \right) \\
        &\leq \frac{15\tilde{\eta}\Theta_{r-1}}{169} + \frac{4\tilde{\eta}\mathbb{E}||\nabla f(\boldsymbol{x})||^2}{169\eta_g^2} + \frac{15\tilde{\eta}^2\beta\hat{\zeta}_{1-p}^2}{26\eta_g^2} +  \frac{30\tilde{\eta}^2\beta\sigma_p^2}{13K\eta_g^2}+\frac{3\tilde{\eta}^2\beta\sigma^2}{26K\eta_g^2} \,.
    \end{split}
\end{equation}
The above step uses $\tilde{\eta} \leq \frac{1}{26\beta}$

\begin{equation}
    9\beta\tilde{\eta}^2\Theta_r \leq 9\beta\tilde{\eta}^2\beta_p^2\mathcal{E}_r \,.
\end{equation}

Adding the above three equations together:
\begin{equation}
    \begin{split}
        \mathbb{E}[f(\boldsymbol{x}+\Delta \boldsymbol{x})] + 9\beta\tilde{\eta}^2\Theta_r &\leq \mathbb{E}[f(\boldsymbol{x})] + 9\beta\tilde{\eta}^2\Theta_{r-1} \\
        &\quad + (4\tilde{\eta}^2\beta + \frac{4\tilde{\eta}}{169\eta_g^2} - \frac{\tilde{\eta}}{2})\mathbb{E}||\nabla f(\boldsymbol{x})||^2 \\
        &\quad + (4\tilde{\eta}^2\beta^3 + 9\beta\beta_p^2\tilde{\eta}^2-\frac{\tilde{\eta}\beta^2}{2})\mathcal{E}_r \\
        &\quad + (\frac{15\tilde{\eta}}{169} - \beta\tilde{\eta}^2)\Theta_{r-1} \\
        &\quad + \frac{\tilde{\eta}^2\beta\sigma^2}{KN}(\frac{3N}{26\eta_g^2} + 3) + \frac{\tilde{\eta}^2\beta\sigma_p^2}{KN}(\frac{30N}{13\eta_g^2} + 6) + \frac{15\tilde{\eta}^2\beta\hat{\zeta}_{1-p}^2}{26\eta_g^2} \\
        &\leq \mathbb{E}[f(\boldsymbol{x})] + 9\beta\tilde{\eta}^2\Theta_{r-1} -\tilde{\eta}\mathbb{E}||\nabla f(\boldsymbol{x})||^2\\
        &\quad +\frac{\tilde{\eta}^2\beta\sigma^2}{KN}(\frac{3N}{26\eta_g^2} + 3) + \frac{\tilde{\eta}^2\beta\sigma_p^2}{KN}(\frac{30N}{13\eta_g^2} + 6) + \frac{15\tilde{\eta}^2\beta\hat{\zeta}_{1-p}^2}{26\eta_g^2} \\ 
        &\leq \mathbb{E}[f(\boldsymbol{x})] + 9\beta\tilde{\eta}^2\Theta_{r-1} -\tilde{\eta}\mathbb{E}||\nabla f(\boldsymbol{x})||^2 + \frac{\tilde{\eta}^2\beta\sigma^2}{KN}(\frac{63N}{26\eta_g^2} + 9) + \frac{15\tilde{\eta}^2\beta\hat{\zeta}_{1-p}^2}{26\eta_g^2}\,.
    \end{split}
\end{equation}

From the first equation, we see $\tilde{\eta} \leq \frac{1}{26\beta}$ implies $(4\tilde{\eta}^2\beta^3 + 9\tilde{\eta}^2\beta\beta_p^2 - \frac{\tilde{\eta}\beta^2}{2}) \leq 0$, $4\tilde{\eta}^2\beta + \frac{4\tilde{\eta}}{169\eta_g^2} - \frac{\tilde{\eta}}{2} \leq 0$, and $\frac{15\tilde{\eta}}{169} - \beta\tilde{\eta}^2 \leq 0$. The last step uses the inequality $\sigma_p^2\leq \sigma^2$

\end{proof}

\vspace*{11mm}

\textbf{Convergence rate:} Average the above equation over $R$ rounds:

\begin{equation}
    \begin{split}
        \frac{1}{R}\sum_{r=1}^R \mathbb{E}||\nabla f(\boldsymbol{x}^r)||^2 &= \frac{\mathbb{E}[f(\boldsymbol{x}^R)] - \mathbb{E}[f(\boldsymbol{x}^0)]}{\tilde{\eta}R} + 9\beta\tilde{\eta}\Theta_0 + \frac{\tilde{\eta}\beta\sigma^2}{KN}(\frac{63N}{26\eta_g^2}+9) + \frac{15\tilde{\eta}\beta\hat{\zeta}_{1-p}^2}{26\eta_g^2}  \\
        &\leq \sigma\sqrt{\frac{\mathbb{E}[f(\boldsymbol{x}^R)] -\mathbb{E}[f(\boldsymbol{x}^0)]}{KNR}}\sqrt{\beta\left(\frac{63N}{\eta_g^2}+9\right)} + \hat{\zeta}_{1-p}\sqrt{\frac{\mathbb{E}[f(\boldsymbol{x}^R)] -\mathbb{E}[f(\boldsymbol{x}^0)]}{R}}\sqrt{\frac{15\beta}{26\eta_g^2}} \\
        &\quad + \frac{26\beta(\mathbb{E}[f(\boldsymbol{x}^R)] -\mathbb{E}[f(\boldsymbol{x}^0)])}{R}\,.
    \end{split}
\end{equation}

Note that if we initialize $\boldsymbol{c}_i^0 = g_i(\boldsymbol{x}^0), \Theta_0=\frac{1}{N}\sum_{i=1}^N\mathbb{E}||\mathbb{E}[\boldsymbol{c}_i^0] - \nabla f_i(\boldsymbol{x}^0)||^2 = 0$. The last step follows from using a stepsize $\tilde{\eta} = \text{min}\left(\frac{1}{26\beta}, \sqrt{\frac{\mathbb{E}[f(\boldsymbol{x}^R)] -\mathbb{E}[f(\boldsymbol{x}^0)]}{R}}\sqrt{\frac{1}{\frac{\beta\sigma^2}{KN}(\frac{63N}{\eta_g^2}+9) + \frac{15\beta\hat{\zeta}_{1-p}^2}{26\eta_g^2}}}\right)$ and Lemma.~\ref{lemma:tune_stepsize}.

Therefore, for non-convex functions:
\begin{equation}
    \mathbb{E}||\nabla f(\bar{\boldsymbol{x}}^R)||^2\leq \mathcal{O}\left(\frac{\sigma\sqrt{F}}{\sqrt{KNR}}\sqrt{\beta\left(\frac{N}{\eta_g^2} + 1\right)} + \frac{\hat{\zeta}_{1-p}\sqrt{F}}{\sqrt{R}}\sqrt{\frac{\beta}{\eta_g^2}} + \frac{\beta F}{R} \right)\,.
\end{equation}

\section{Extra experimental setup and results}
\label{sec:appendix_experiment}
\subsection{Additional experimental setups}

\textbf{Data distribution}
We follow the procedure as described in~\cite{DBLP:conf/nips/LinKSJ20} to simulate the data heterogeneity scenario. There are around $5000$ images per client. We use $1\% (500)$ of the images as the validation dataset to tune the hyperparameters (learning rate and schedule). Fig.~\ref{fig:distribution_of_data} shows data distribution across clients using CIFAR10 and CIFAR100. When $\alpha=0.1$, some clients may only have data from a single class. 
\begin{figure}[ht!]
    \centering
    \includegraphics{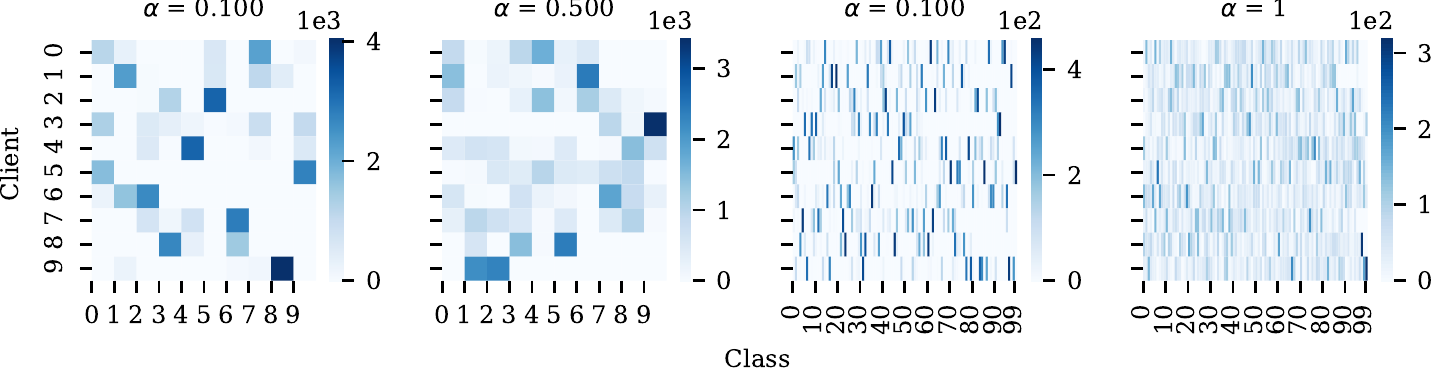}
    \caption{Different levels of data heterogeneity. The left two figures are from CIFAR10 and the right two figures are from CIFAR100.}
    \label{fig:distribution_of_data}
\end{figure}

\textbf{Conformal prediction}
Let $d_i$ denotes a data point and $l_i$ denotes the corresponding label. Conformal prediction aims to produce a predictive set $\mathcal{C}_{\kappa}(d_i)$ such that:
\begin{equation}
    P(l_i \in \mathcal{C}_{\kappa}(d_i)) \geq 1 - \kappa
\end{equation}
where $\kappa \in (0, 1)$ specifies the desired converge level. To find the predictive set $\mathcal{C}$, we use a threshold $\tau$ on the predicted probability to indicate which predictions are included in the prediction set. Specifically, we assume that the server can have a small portion of a dataset in the same domain as the datasets in each client (e.g. the validation dataset). As we only suggest conformal prediction as an effective tool when the task is sensitive and producing a wrong prediction is dangerous (e.g. chemical hazards detection), we think it is reasonable to make such an assumption to guarantee the performance. 

\textbf{Architecture}
We use two types of neural networks in our paper: VGG-11 and ResNet8. We simply adopt the architectures that are used in~\cite{DBLP:conf/nips/LinKSJ20}. Detailed information about these two architectures can be found at \url{https://github.com/epfml/federated-learning-public-code}.

\textbf{Hyperparameters} The learning rate and learning rate schedule are shown in the following two tables. Most of the experiments use a learning rate of 0.1 with a constant schedule, which means that each client always uses a learning rate of 0.1 along the communication rounds. To stabilize the training procedure, we apply momentum with a factor of 0.9 on the block of weights that are not variance-reduced. We use 10 clients with full participation. The number of local epochs per round is 10. For FedProx~\cite{DBLP:journals/corr/abs-1812-06127}, we tune the temperature parameter $\mu$ from $\{0.1, 0.5, 1.0\}$. For FedDyn~\cite{DBLP:journals/corr/abs-2111-04263}, the penalization parameter ($\alpha$ as used in~\cite{DBLP:journals/corr/abs-2111-04263}) is 0.001 for CIFAR10 and 0.01 for CIFAR100. We tune the centralised learning experiment with constant learning rate and momentum parameter from $\{0.9, 0.95\}$. 

\begin{table}[ht!]
\caption{The learning rate and learning rate schedule for CIFAR10 experiments where $c$ represents \texttt{constant} and $m$ represents multi-step decay}
\resizebox{0.85\textwidth}{!}
{\begin{minipage}{\textwidth}
\begin{tabular}{@{}lcccccccccccc@{}}
\toprule
 & \multicolumn{6}{c}{VGG-11} & \multicolumn{6}{c}{ResNet8} \\ \midrule
 & \multicolumn{2}{c}{FedAvg} & \multicolumn{2}{c}{SCAFFOLD} & \multicolumn{2}{c}{FedPVR} & \multicolumn{2}{c}{FedAvg} & \multicolumn{2}{c}{SCAFFOLD} & \multicolumn{2}{c}{FedPVR} \\
 & $\alpha=0.1$  & $\alpha=0.5$ &$\alpha=0.1$  & $\alpha=0.5$   & $\alpha=0.1$ & $\alpha=0.5$   & $\alpha=0.1$ & $\alpha=0.5$   &$\alpha=0.1$  & $\alpha=0.5$   &$\alpha=0.1$  & $\alpha=0.5$   \\
LR &0.05  &0.1  &0.05  &0.1  &0.05  &0.1  & 0.1 &0.2  &0.1  &0.3  &0.1  & 0.3 \\
LR-schedule &c  &c  &m  & c & c &c  &c &c  &c  &c  &c  & c \\ \bottomrule
\end{tabular}
\end{minipage}}
\end{table}

\begin{table}[ht!]
\caption{The learning rate and learning rate schedule for CIFAR100 experiments where $c$ represents \texttt{constant} and $m$ represents \texttt{multi-step decay} and $cos$ represents \texttt{cosine-decay}}
\resizebox{0.85\textwidth}{!}
{\begin{minipage}{\textwidth}
\begin{tabular}{@{}lcccccccccccc@{}}
\toprule
 & \multicolumn{6}{c}{VGG-11} & \multicolumn{6}{c}{ResNet8} \\ \midrule
 & \multicolumn{2}{c}{FedAvg} & \multicolumn{2}{c}{SCAFFOLD} & \multicolumn{2}{c}{FedPVR} & \multicolumn{2}{c}{FedAvg} & \multicolumn{2}{c}{SCAFFOLD} & \multicolumn{2}{c}{FedPVR} \\
 & $\alpha=0.1$  & $\alpha=1.0$ &$\alpha=0.1$  & $\alpha=1.0$ & $\alpha=0.1$ & $\alpha=1.0$  & $\alpha=0.1$ & $\alpha=1.0$   &$\alpha=0.1$  & $\alpha=1.0$   &$\alpha=0.1$  & $\alpha=1.0$   \\
LR &0.1  &0.05  &0.1  &0.1  &0.1  &0.1  & 0.1 & 0.1 &0.2  & 0.2 &0.1  &0.1  \\
LR-schedule &c  & c&c  &c& c &c  &c &c  &cos  &c  &c  & c \\ \bottomrule
\end{tabular}
\end{minipage}}
\end{table}

\subsection{Additional experimental results}

\begin{figure}[ht!]
    \centering
    \includegraphics{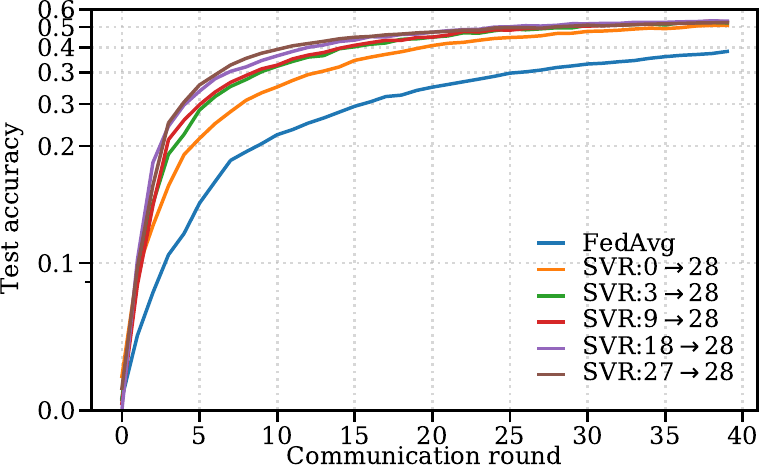}
    \caption{Influence of using variance reduction on layers that start from different positions in a neural network on the learning speed. \texttt{SVR:0$\rightarrow$28} applies variance reduction on the entire model, which corresponds to \texttt{SCAFFOLD}. \texttt{SVR:27$\rightarrow$28} applies variance reduction from the layer index 27 to 28, which corresponds to our method. The later we apply variance reduction, the better performance speedup we obtain. However, no variance reduction (\texttt{FedAvg}) performs the worst here.}
    \label{fig:resnet_svr}
\end{figure}

\begin{figure}[ht!]
    \centering
    \includegraphics{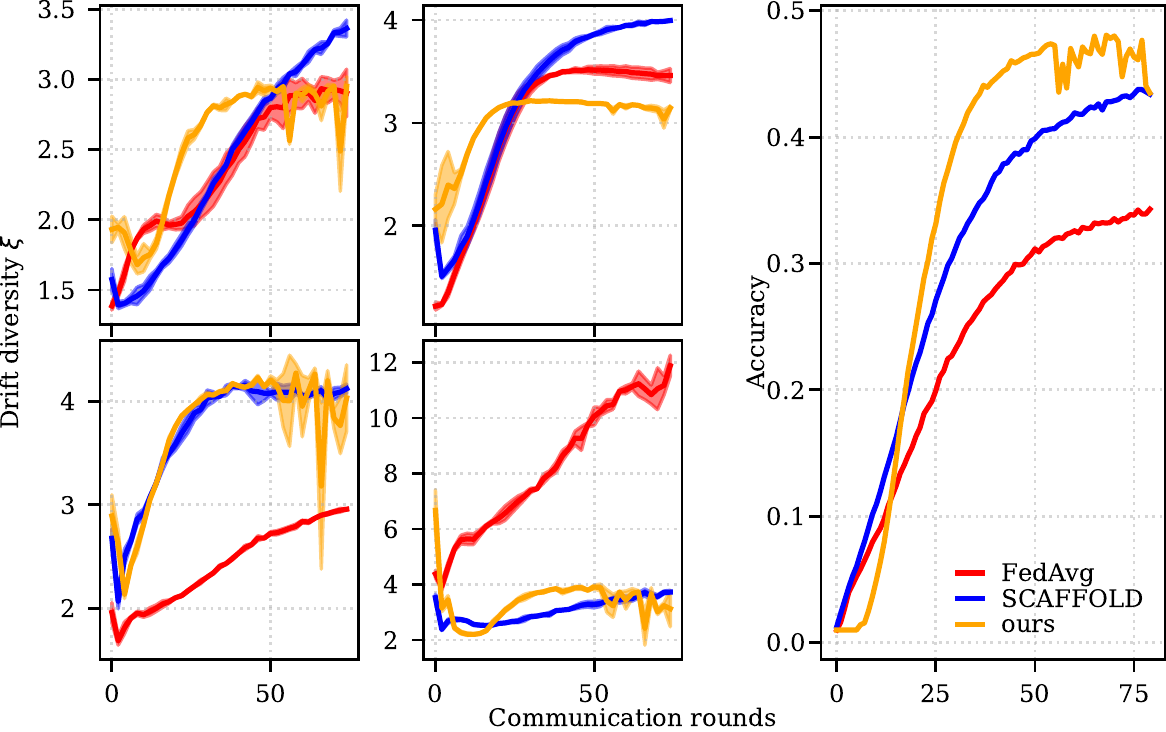}
    \caption{Drift diversity and learning curve for VGG-11 on CIFAR100 with $\alpha=0.1$. Compared to FedAvg, SCAFFOLD and our method can both improve the agreement between the classifiers. Compared to SCAFFOLD, our method results in a higher gradient diversity at the early stage of the communication, which tends to boost the learning speed as the curvature of the drift diversity seem to match the learning curve.}
    \label{fig:my_label}
\end{figure}

\begin{figure}[ht!]
    \centering
    \includegraphics{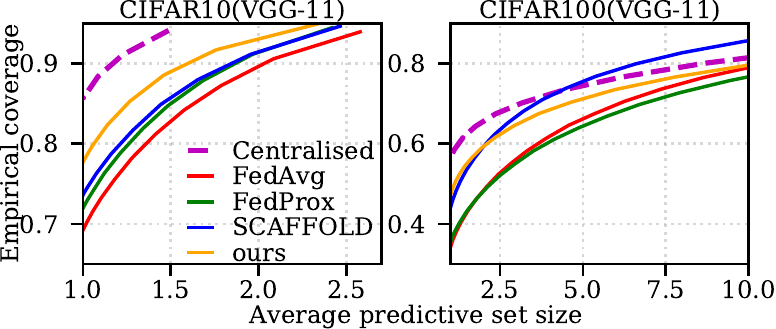}
    \caption{Relation between average predictive size and empirical coverage using centralised empirical coverage instead of centralised Top-1 accuracy.}
    \label{fig:conformal_with_centralised}
\end{figure}

\end{document}